\documentclass[11pt]{article}

\usepackage[preprint]{acl}

\usepackage{times}
\usepackage{latexsym}
\usepackage[T1]{fontenc}
\usepackage[utf8]{inputenc}
\usepackage{microtype}
\usepackage{inconsolata}
\usepackage{graphicx}
\usepackage{subfigure}
\usepackage{enumitem}

\usepackage{amsmath}
\usepackage{amssymb}
\usepackage{amsthm}
\newtheorem{proposition}{Proposition}
\newtheorem{lemma}{Lemma}
\usepackage{booktabs}
\usepackage{multirow}
\usepackage{colortbl}
\usepackage[dvipsnames]{xcolor}
\usepackage{tabularx}
\usepackage{float}
\usepackage[most]{tcolorbox}
\usepackage{pgfplots}
\pgfplotsset{compat=1.18}

\newtcolorbox[auto counter]{dialogbox}[2][]{
enhanced, breakable,
colback=gray!5, colframe=black!40,
fonttitle=\small\bfseries,
title={#2},
left=6pt, right=6pt, top=4pt, bottom=4pt,
#1
}

\title{\textsc{Hyper-ES}: Effective Evolution Strategies for LLM Reasoning\\via Descent Direction Merging}
\author{
Yu Gu$^{3,1,*,\ddagger}$ \quad
Zhi Zheng$^{2,*,\dagger}$ \quad
Yunpeng Ba$^{1}$ \quad
Xialiang Tong$^{4}$ \quad
Mingxuan Yuan$^{4}$ \quad
Zhenkun Wang$^{1,\dagger}$ \\
$^{1}$School of Automation and Intelligent Manufacturing, 
Southern University of Science and Technology, China \\
$^{2}$School of Computing, National University of Singapore, Singapore \\
$^{3}$School of Intelligence Science and Technology, 
Nanjing University, China \\
$^{4}$Noah's Ark Lab, Huawei Technologies Ltd., China\\
$^{\ddagger}$Work done while interning at Southern University of Science and Technology.\\
$^{*}$Equal contribution. \quad $^{\dagger}$Co-corresponding authors.
}

\begin{document}
\maketitle

\begin{abstract}
Evolution Strategy (ES) is a promising alternative to gradient-based fine-tuning for resource-constrained Large Language Model (LLM) reasoning. However, directly applying ES to billion-parameter LLMs is \textbf{highly ineffective}. In such high-dimensional parameter spaces, most random perturbations are nearly orthogonal to useful update directions, leading to unstable optimization.
We propose \textsc{Hyper-ES}, a subspace-based ES framework that avoids the weakness of ES in full-parameter search while \textbf{exploiting its strength in low-dimensional optimization}. Instead of asking ES to discover useful directions from random perturbations in the LLM parameter space, \textsc{Hyper-ES} first performs a small number of inexpensive gradient-based fine-tuning runs to obtain descent directions. Although each direction may provide only a limited improvement on its own, their span forms a compact adaptation subspace that captures useful reasoning updates. \textsc{Hyper-ES} then applies CMA-ES to optimize layer-wise DARE--TIES merging coefficients within this subspace, allowing ES to search over combinations of meaningful descent directions rather than over arbitrary full-model perturbations.
We evaluate \textsc{Hyper-ES} on three Qwen2.5-Instruct and DeepSeek-R1-Distill backbones across six mathematical reasoning datasets. Results show that \textsc{Hyper-ES} consistently outperforms GRPO-LoRA by 1\% while requiring 10\% fewer space-consuming gradient updates.\footnote{Code at \url{https://github.com/kuangrepi/Hyper-ES}.}
\end{abstract}

\section{Introduction}

Large Language Models (LLMs) can demonstrate incredible capabilities in mathematical reasoning \cite{li2025system}, especially when equipped with the chain-of-thought (CoT) prompting techniques \cite{wei2022chain} and further refined with gradient-based supervised fine-tuning (SFT) \cite{zheng2025reasoning} or reinforcement learning (RL) methods \cite{shao2024deepseekmath}. However, although these methods achieve significant improvements in reasoning ability, the gradient backpropagation process within them leads to significantly higher time and memory consumption \cite{liu2025efficient}, which hinders the fine-tuning of LLMs for reasoning in resource-constrained scenarios \cite{park2025mobilerag}.

\begin{figure}[H]\vspace{-5pt}
    \centering
    \subfigure[Gradient-based GRPO for LLM Reasoning]{\includegraphics[width = 0.5\textwidth]{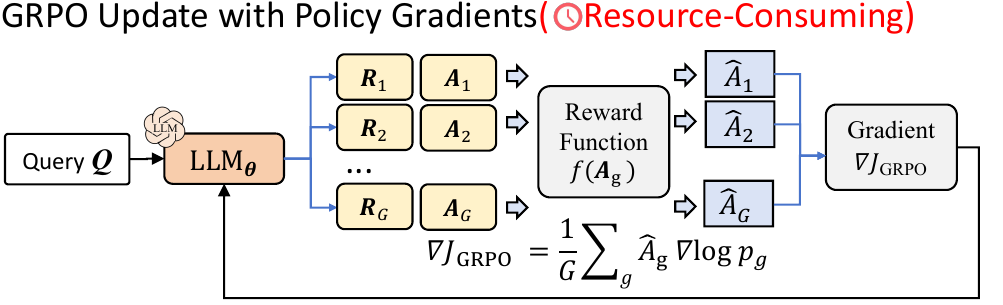}}\vspace{-2pt}
    \subfigure[Gradient-free ES for LLM Reasoning]{\includegraphics[width = 0.5\textwidth]{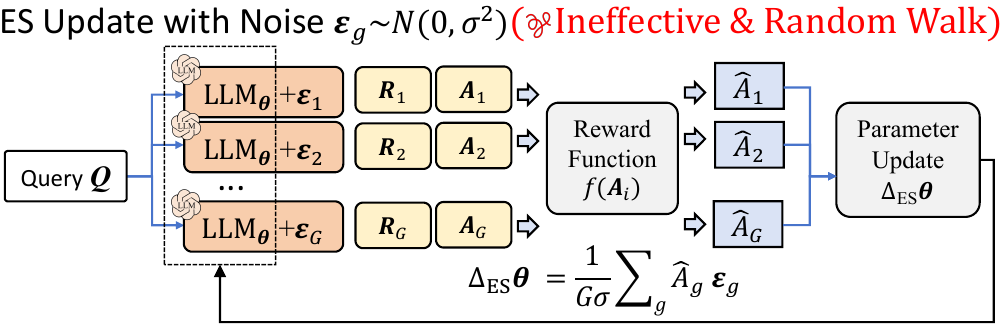}}\vspace{-2pt}
    \subfigure[Hyper-ES (Ours) for ES over Merging Descent Directions]{\includegraphics[width = 0.5\textwidth]{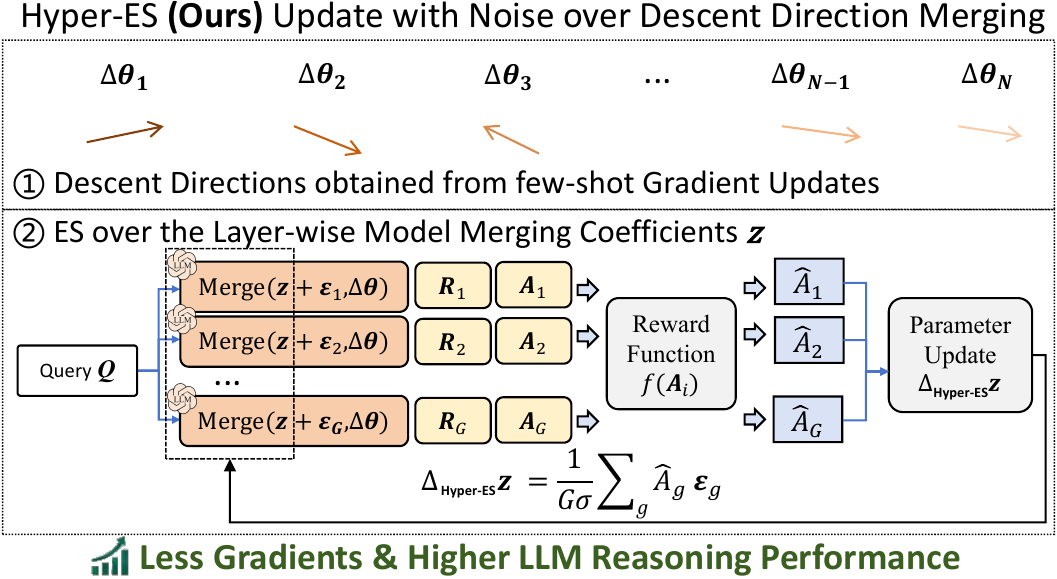}}\vspace{-5pt}
    \caption{(a) GRPO updates model parameters with policy gradients, which is usually resource-consuming. (b) ES offers a gradient-free method for LLM reasoning, suitable for resource-constrained scenarios, but will lead to significant inefficiency \& out-of-control random walk. (c) Hyper-ES replaces full-parameter exploration with low-dimensional search over coefficients of fast-obtained descent directions, leading to Higher LLM Reasoning Performance with Fewer Gradient Updates.}\label{fig:figure1}
\end{figure}\vspace{-5pt}

Recent Evolution Strategy (ES) methods have emerged as a promising alternative to gradient-based fine-tuning for improving LLM reasoning ability~\cite{qiu2025evolution,sarkar2025evolution,sun2026essam}. Unlike gradient-based methods (e.g., Group Relative Policy Optimization (GRPO), shown in Figure \ref{fig:figure1} (a)) that backpropagate through the model, as shown in Figure \ref{fig:figure1} (b), ES only perturbs parameters, evaluates each candidate with a verifiable reward, and updates the parameters toward better-performing regions, demonstrating \textbf{10$\times$ space-efficiency}~\citep{sun2026essam}. However, its direct application to LLMs is severely limited by the dimensionality. In billion-parameter spaces, almost all randomly sampled perturbations lie in directions that are irrelevant to task improvement. With a limited population size, ES therefore receives extremely weak directional signals, making full-parameter search sample-inefficient and unstable~\cite{hoy2026matching,abdi2026evolutionary}.

To keep the efficiency of ES while avoiding its drawback in high-dimensional direction discovery, we decouple the exploration of update directions from derivative-free optimization. Rather than asking ES to find useful descent directions through random perturbations in the full LLM parameter space, \textsc{Hyper-ES} first constructs $N$ few-shot LoRA directions from a shared base model by running fewer than ten GRPO update steps on different data subsets. Although these preliminary updates are cheap and individually limited, their LoRA deltas provide task-relevant descent directions that better align with reasoning improvement than isotropic random noise. \textsc{Hyper-ES} then treats these deltas as basis directions and applies the Covariance Matrix Adaptation Evolution Strategy (CMA-ES) \cite{hansen2001cmaes} to optimize layer-wise DARE--TIES merging coefficients over their combinations. \textsc{Hyper-ES} reformulates full-parameter adaptation as a low-dimensional, structured coefficient search with only hundreds of parameters, allowing ES to preserve its memory-light and parallelizable nature while operating in the regime where it is most effective.
Empirically, we implement \textsc{Hyper-ES} on Qwen2.5-0.5B-Instruct, Qwen2.5-1.5B-Instruct, and DeepSeek-R1-Distill-1.5B. Across four arithmetic reasoning benchmarks, including GSM8K, GSM-Hard, SVAMP, and MultiArith, as well as two more challenging mathematical reasoning benchmarks, AMC23 and MATH-500, \textsc{Hyper-ES} improves over model-merging baselines and slightly but consistently outperforms single-stage GRPO. In particular, it achieves up to a 1\% performance gain while requiring 10\% fewer expensive backpropagation steps. Our contributions are as follows:
\begin{itemize}[leftmargin=*]
\item We propose \textsc{Hyper-ES}, a descent-direction-assisted ES framework that converts full-parameter search into a compact layer-wise coefficient search over task-relevant LoRA updates.

\item \textsc{Hyper-ES} uses few-shot GRPO updates to construct informative descent directions, allowing ES to avoid high-dimensional random exploration while retaining its efficiency.

\item Across \textbf{six} math reasoning benchmarks and \textbf{three} LLM backbones, \textsc{Hyper-ES} consistently improves over model-merging baselines and surpasses single-stage GRPO with fewer expensive backpropagation steps.

\end{itemize}

\section{Preliminaries}
\label{sec:background}

\subsection{LLM Reasoning} \label{languageCoT}
The language reasoning process addresses a given question $\boldsymbol{Q}=(q_1,\ldots,q_{|\boldsymbol{Q}|})$ by first generating a series of CoT language reasoning tokens $\boldsymbol{R}=(r_1,\ldots,r_{|\boldsymbol{R}|})$, followed by answer tokens $\boldsymbol{A}=(a_1,\ldots,a_{|\boldsymbol{A}|})$. Both reasoning and answer tokens are produced according to the next-token prediction policy $\pi_\theta$ of LLMs as follows:
\begin{equation}
    \begin{aligned}
        p(\boldsymbol{R},\boldsymbol{A}|\boldsymbol{Q}) = &\prod_{t=1}^{|\boldsymbol{R}|}\pi_{\theta}(r_t|[\boldsymbol{Q},\boldsymbol{r}_{1:t-1}])\\
        & \qquad \prod_{t=1}^{|\boldsymbol{A}|}\pi_{\theta}(a_t|[\boldsymbol{Q},\boldsymbol{R},\boldsymbol{a}_{1:t-1}]),
    \end{aligned}\label{problanguage}\notag
\end{equation}
where $\boldsymbol{r}_{1:t-1}=(r_1,\ldots,r_{t-1})$ and $\boldsymbol{a}_{1:t-1}=(a_1,\ldots,a_{t-1})$; $[\cdot,\cdot]$, $[\cdot,\cdot,\cdot]$ denote concatenation.

\begin{figure*}[htbp]
    \begin{equation}
        \begin{aligned}
            &\mathcal{J}_{\text{GRPO}}(\theta) = \frac{1}{G} \mathbb{E}_{\{\boldsymbol{R},\boldsymbol{A}\}_{g=1}^G \sim p(\cdot,\cdot|\boldsymbol{Q})}\Bigg[\sum_{g=1}^G \frac{1}{\left| \boldsymbol{R}_g \right|+\left| \boldsymbol{A}_g \right|} \sum_{t=1}^{\left| \boldsymbol{R}_g \right|+\left| \boldsymbol{A}_g \right|} \Big( \min \left( p_{g,t} \hat{A}_{g}, \text{clip}(p_{g,t}, 1 - \epsilon, 1 + \epsilon)\hat{A}_{g} \right) \Bigg]\\
            &\quad\hat{\boldsymbol{A}}_{g}=\frac{f(\boldsymbol{A}_{g})-\text{mean}(f(\boldsymbol{A}))_{g=1}^G}{\text{std}(f(\boldsymbol{A}))_{g=1}^G},\qquad p_{g,t} =
            \begin{cases}
                \frac{\pi_{\theta}(a_{g,t} | [\boldsymbol{Q}, \boldsymbol{r}_g, (a_{g,1}, \ldots, a_{g,t-1})])}{\pi_{\theta_{\text{old}}}(a_{g,t} | [\boldsymbol{Q}, \boldsymbol{r}, (a_{g,1}, \ldots, a_{g,t-1})])} & \text{if } t > |\boldsymbol{R}_g| \\[1pt]
                \frac{\pi_{\theta}(\boldsymbol{r}_{g,t} | [\boldsymbol{Q}, (\boldsymbol{r}_{g,1}, \ldots, \boldsymbol{r}_{g,t-1})])}{\pi_{\theta_{\text{old}}}(\boldsymbol{r}_{g,t} | [\boldsymbol{Q}, (\boldsymbol{r}_{g,1}, \ldots, \boldsymbol{r}_{g,t-1})])} & \text{if } t \leq |\boldsymbol{R}_g|.
            \end{cases}
        \end{aligned}\label{grpo}
    \end{equation}\vspace{-5pt}
\end{figure*}

\paragraph{RL Fine-tuning for LLM Reasoning} RL methods—such as Group Relative Policy Optimization (GRPO) \cite{shao2024deepseekmath}, Dr. GRPO \cite{liu2025understanding}, DAPO \cite{yu2025dapo}—sample multiple candidate CoTs $[\boldsymbol{R}, \boldsymbol{A}]$ per question and assess each with a reward reflecting the answer quality $\boldsymbol{A}$. For example, in standard GRPO \cite{shao2024deepseekmath} (shown in Eq. \eqref{grpo}), $G$ candidate CoTs $\{\boldsymbol{R},\boldsymbol{A}\}_{g=1}^G$ are generated for each $\boldsymbol{Q}$, and the objective is updated according to the relative advantage $\hat{\boldsymbol{A}}_g$ within these $G$ samples. These RL-based approaches often significantly surpass Supervised Fine-tuning (SFT) in mathematical reasoning tasks, making them the top choice of LLM Reasoning. However, due to the long rollout trajectory (usually thousands of tokens), updating the parameter through backpropagation consumes a huge amount of time and space, which is unaffordable for resource-constrained application scenarios.

\subsection{Evolution Strategy for LLM Reasoning}
\label{sec:es_prelim}
\label{sec:random_walk}

\paragraph{ES Fine-tuning for LLM Reasoning} Direct ES starts from the pretrained parameters $\theta\in\mathbb{R}^d$ and repeatedly builds a population of nearby models. At iteration $t$, the current model is $\theta_t$, $\theta_0=\theta$. ES samples $G$ random perturbations from the Gaussian distribution and forms a population as follows:
\begin{equation}
    \begin{aligned}
        \theta_{t,k}
        =
        \theta_t
        +
        &\epsilon_{t,k},
        \quad
        \epsilon_{t,k}\sim\mathcal{N}(0,\sigma ^2I_d),\\
        & k=1,\ldots,G,
    \end{aligned}
    \label{eq:llm_scale_sampling}
\end{equation}
where $\epsilon_{t,k}$ is a random direction in the parameter space and $\sigma$ controls the degree of the perturbation. Each candidate model $\theta_{t,k}$ is evaluated by a scalar fitness $f(\theta_{t,k})$, such as a verifiable reward or validation accuracy (e.g., the answer-based reward $\hat{\boldsymbol{A}}_g$ in Eq. \eqref{grpo}). These fitness values are then used to estimate an update direction:
\begin{equation}
    \widehat{g}_{\mathrm{ES}}(\theta_t)
    =
    \frac{1}{G}
    \sum_{k=1}^{G}
    f(\theta_{t,k})\epsilon_{t,k}.
    \label{eq:es_update_estimator}\notag
\end{equation}
The current parameters are updated as follows \cite{salimans2017evolution,qiu2025evolution,sun2026essam}:
\begin{equation}
    \theta_{t+1}
    =
    \theta_t
    +
    \eta_t
    \widehat{g}_{\mathrm{ES}}(\theta_t).
    \label{eq:llm_scale_update}\notag
\end{equation}
Here $\eta_t$ is the ES learning rate. Due to ES only requiring model forward rollout and evaluations, instead of backpropagation, ES-based LLM fine-tuning methods usually consume less time and space compared to gradient-based methods (e.g., GRPO). As mentioned in \citet{sun2026essam}, ES consumes $10\times$ less GPU memory usage, making it suitable for resource-constrained scenarios.

However, when $\theta_t$ is of thousands, even billions of dimensions, every $\epsilon_{t,k}$ is a random vector in a billion-dimensional space.

\paragraph{Problem 1: directional discovery failure.}
The first drawback of direct ES is that it can hardly discover descent directions by random search. In high dimensions, a random perturbation is likely to be nearly orthogonal to any useful descent direction. This remains true considering ES populations, because the population size grows far more slowly than the ambient parameter dimension.

\begin{lemma}[High-dimensional angles concentrate near orthogonality]
    \label{lem:near_orthogonal}
    Let $g^*=-\nabla_\theta \mathcal{L}(\theta_t)$ be any fixed nonzero descent direction. Let $u_1,\ldots,u_G$ be $G=G(d)$ independently sampled unit perturbation directions in $\mathbb{R}^d$, and define the acute angle
    \begin{equation}
        \alpha_k
        =
        \arccos
        \left(
        \left|
        \left\langle u_k,\frac{g^*}{\|g^*\|}\right\rangle
        \right|
        \right)
        \in[0,\pi/2].
    \end{equation}
    Given a limited population (i.e., $\log G=o(d)$),
    \begin{equation}
        \max_{1\leq k\leq G}
        \left|
        \left\langle u_k,\frac{g^*}{\|g^*\|}\right\rangle
        \right|
        \xrightarrow[d\to\infty]{p}
        0,
        \label{eq:alignment_background}
    \end{equation}
    or equivalently,
    \begin{equation}
        \min_{1\leq k\leq G}\alpha_k
        \xrightarrow[d\to\infty]{p}
        \frac{\pi}{2}.
    \end{equation}
\end{lemma}

The proof is provided in Appendix~\ref{app:random_walk_proofs}. For billion-parameter LLMs, Lemma~\ref{lem:near_orthogonal} implies that a practical ES population is dominated by perturbations whose angles to the loss-reducing direction are close to $90^\circ$. The population can contain models with different rewards, but those reward differences do not imply that ES has found a direction close to the optimal update. Instead, when the useful projection onto $g^*$ is tiny, the update induced by population reweighting is largely formed from components orthogonal to $g^*$.

\paragraph{Problem 2: orthogonal random-walk.}
The second drawback is the accumulation of these orthogonal updates. Even if each individual orthogonal component is uninformative for loss reduction, repeated ES updates can accumulate in the parameter space and cause the model to drift away from the pretrained initialization.

\begin{lemma}[Irrelevant ES updates accumulate parameter drift]
    \label{lem:random_walk}
    Let $r_t=\widehat{g}_{\mathrm{ES}}(\theta_t)-\mathrm{Proj}_{g^*}\widehat{g}_{\mathrm{ES}}(\theta_t)$ denote the component of the ES update estimate orthogonal to $g^*$. Suppose the useful component is negligible and the orthogonal components have a nonzero second moment, $\mathbb{E}\|r_t\|_2^2=\rho^2>0$. Then, after $T$ ES steps, the accumulated orthogonal parameter displacement satisfies
    \begin{equation}
        \mathbb{E}\!\left[
        \left\|
        \sum_{t=0}^{T-1}\eta_t r_t
        \right\|_2^2
        \right]
        \approx
        \rho^2\sum_{t=0}^{T-1}\eta_t^2.
        \label{eq:no_descent_progress}
    \end{equation}
    Thus, ES can induce a large parameter-space drift even when it makes little progress along the optimal direction, leading to unstable optimization.
\end{lemma}
\begin{figure}[t]
    \centering
    \subfigure[ES leads to wrong directions]{\includegraphics[width = 0.28\textwidth]{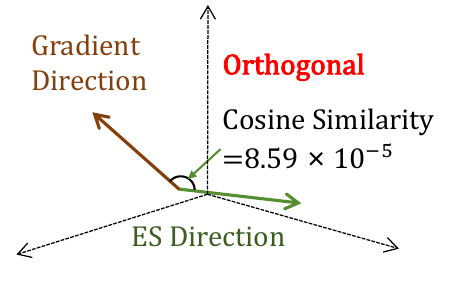}}
    \subfigure[ES leads to random walk]{\includegraphics[width = 0.32\textwidth]{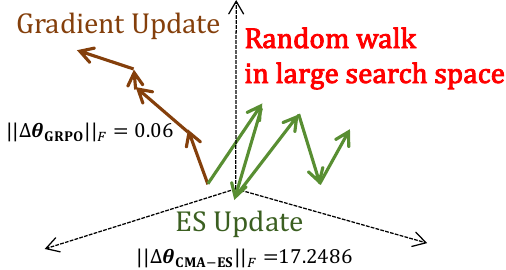}}
    \caption{Evidence for Lemmas~\ref{lem:near_orthogonal} and~\ref{lem:random_walk}. We show the detailed raw data for this figure in Appendix \ref{app:es_drift_evidence}.}\label{fig:figure15}
\end{figure}

We show evidence for these two in Figure \ref{fig:figure15} and the proof is provided in Appendix~\ref{app:random_walk_proofs}. Together, Lemmas~\ref{lem:near_orthogonal} and~\ref{lem:random_walk} explain why direct ES becomes \textbf{ineffective and unstable} at the LLM scale. The population is likely to consist of perturbations nearly orthogonal to the optimal descent direction, while repeated updates along unrelated directions can still move the model far from the pretrained parameters. This geometric mismatch is consistent with the observation of \citet{hoy2026matching}.

This motivates two requirements for ES-based LLM reasoning fine-tuning. \textbf{1)} The search dimension must be small enough for population-based adaptation to be meaningful, and \textbf{2)} the reduced search space must contain task-relevant directions. \textsc{Hyper-ES} satisfies both requirements by constructing a low-dimensional search space from descent directions.

\section{Method: Hyper-ES}
\label{sec:method}

\textsc{Hyper-ES} follows the pipeline in Figure~\ref{fig:figure1}(c): it first prepares a small set of task-relevant descent directions, and then runs ES over combinations of these directions rather than over the full model parameters. The role of ES is therefore changed from direction discovery to direction selection. Figure~\ref{fig:figure2} gives the implementation details behind this design: the left panel shows how the descent directions are obtained, and the right panel shows how a search variable is decoded into a merged model.

\begin{figure*}[htbp]
    \centering
    \subfigure[Descent Directions from few-shot Gradient-based Updates]{\includegraphics[width = 0.30\textwidth]{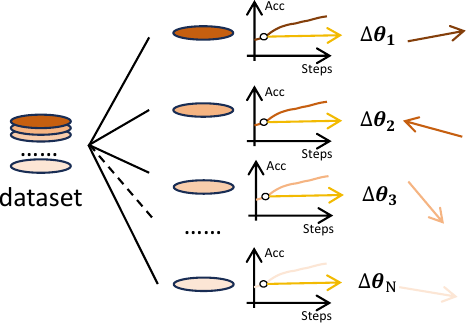}}\quad
    \subfigure[Model Merging Process with ES-sampled coefficients $\boldsymbol{z}$]{\includegraphics[width = 0.66\textwidth]{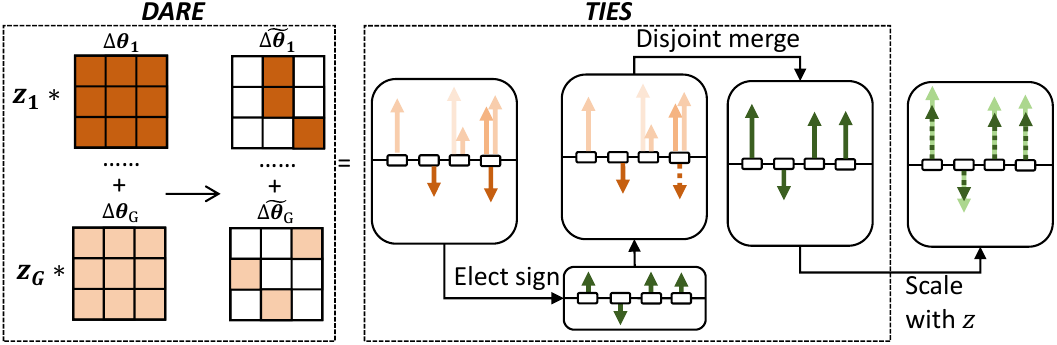}}
    \caption{Implementation details of \textsc{Hyper-ES}. (a) A small number of few-shot GRPO updates extract task-relevant descent directions from different data subsets. (b) A search variable $\boldsymbol{z}$ is decoded into layer-wise DARE--TIES merging coefficients that combine these directions into a final model.}\label{fig:figure2}
\end{figure*}

\subsection{Start-up: Preparing Descent Directions}
\label{sec:direction_pool}

\textsc{Hyper-ES} starts from a pretrained model with parameters $\boldsymbol{\theta}$. As illustrated in Figure~\ref{fig:figure2}(a), we partition the training data into $N$ subsets $\{\mathcal{D}_i\}_{i=1}^{N}$ and run one short LoRA-based GRPO update on each subset, typically 7 update steps. These short runs are not meant to produce strong, standalone models. Their purpose is to cheaply reveal descent directions that already contain a task signal, so that the later ES stage does not need to find such directions through random perturbations in the original parameter space.

Let $\boldsymbol{\theta}_i$ denote the parameters obtained after the $i$-th short GRPO run. For LLMs with $L$ layer-wise modules (which means, attention or MLP), indexed by $\ell\in\{1,\ldots,L\}$, we define the descent direction as follows:
\begin{equation}
    \Delta\boldsymbol{\theta}_i^{(\ell)}
    =
    \boldsymbol{\theta}_i^{(\ell)} - \boldsymbol{\theta}_0^{(\ell)},
    \label{eq:task_vector_main}\notag
\end{equation}
where $\boldsymbol{\theta}_0^{(\ell)}$ and $\boldsymbol{\theta}_i^{(\ell)}$ are the base parameters and the parameters after the $i$-th short run for layer-wise group $\ell$. When the group corresponds to a LoRA-updated weight matrix with rank $r$ and scale $\alpha$, this delta is computed without materializing a full fine-tuned model as follows:
\begin{equation}
\begin{aligned}
    &\qquad\Delta\boldsymbol{\theta}_i^{(\ell)}
    =
    \frac{\alpha}{r}\boldsymbol{B}_i^{(\ell)}\boldsymbol{A}_i^{(\ell)}.
    \\ \mathcal{P}
    =&
    \left\{\Delta\boldsymbol{\theta}_i^{(\ell)}:
    i=1,\ldots,N,\; \ell=1,\ldots,L
    \right\}\notag
\end{aligned}
\end{equation}
Then $\mathcal{P}$ is the direction pool, which can provide descent directions for the following merging search of Hyper-ES. 

\subsection{Hyper-ES Search over Directions}
\label{sec:merging_subspace}

After the start-up stage, \textsc{Hyper-ES} keeps the pretrained model and the direction pool fixed. It then searches over a compact variable $\boldsymbol{z}$ that specifies how to combine the cached directions. For each layer group, $\boldsymbol{z}$ contains two types of variables: drop-rate logits controlling DARE \cite{yu2024dare} sparsification, and mixing-weight logits controlling how strongly each direction contributes. For
\begin{equation}
    \boldsymbol{z}
    =
    \{\boldsymbol{z}_\ell\}_{\ell=1}^{L}
    \in \mathbb{R}^{L\times 2N},
    \quad
    \boldsymbol{z}_\ell=(\boldsymbol{a}_{\ell,1:N}, \boldsymbol{b}_{\ell,1:N}),
    \label{eq:z_definition_main}\notag
\end{equation}
where $a_{\ell,i}$ is a drop-rate logit and $b_{\ell,i}$ is a mixing-weight logit for direction $i$ in the layer-wise group $\ell$. These logits are decoded as follows:
\begin{equation}
    p_i^{(\ell)}=\sigma(a_{\ell,i}),
    \qquad
    \omega_i^{(\ell)}=\exp(b_{\ell,i}),
    \label{eq:decode_main}\notag
\end{equation}
where $p_i^{(\ell)}\in(0,1)$ is the DARE drop rate and $\omega_i^{(\ell)}>0$ is the non-negative mixing weight.

\paragraph{From search variables to a model.}
As shown in Figure \ref{fig:figure2} (b), given $\boldsymbol{z}$, \textsc{Hyper-ES} decodes $p_i^{(\ell)}$ and $\omega_i^{(\ell)}$, applies DARE-style sparsification to the directions, resolves sign conflicts with weighted TIES \cite{yadav2023tiesmerging}, and obtains a merged update $\overline{\Delta\boldsymbol{\theta}}^{(\ell)}(\boldsymbol{z}_\ell)$ for each group. The final model is then assembled as follows:
\begin{equation}
    \boldsymbol{\theta}_*^{(\ell)}(\boldsymbol{z})
    =
    \boldsymbol{\theta}_0^{(\ell)}
    +
    \gamma\overline{\Delta\boldsymbol{\theta}}^{(\ell)}(\boldsymbol{z}_\ell),
    \quad \ell=1,\ldots,L,
    \label{eq:merged_weight_main}\notag
\end{equation}
where $\gamma>0$ is a global scale selected by validation search. Figure~\ref{fig:figure2}(b) illustrates this decoding path. Full DARE--TIES equations are provided in Appendix~\ref{app:merging_operator}. Equivalently, with the base parameters $\boldsymbol{\theta}_0$ and direction pool $\mathcal{P}$ fixed, the merging operator is
\begin{equation}
    \mathcal{M}_{\boldsymbol{\theta}_0,\mathcal{P}}:
    \boldsymbol{z}\in\mathbb{R}^{L\times 2N}
    \mapsto
    \boldsymbol{\theta}_*(\boldsymbol{z}),
    \label{eq:merging_operator_main}\notag
\end{equation}
where $\boldsymbol{\theta}_*(\boldsymbol{z})=\mathcal{M}_{\boldsymbol{\theta}_0,\mathcal{P}}(\boldsymbol{z})$ denotes the full model assembled from $\{\boldsymbol{\theta}_*^{(\ell)}(\boldsymbol{z})\}_{\ell=1}^{L}$ and the unchanged base parameters. This construction changes the optimization problem from searching over millions or billions of weights to searching over $L\times 2N$ merging variables. In our settings, the search dimension is $960$ for the 0.5B backbones (with $L=48$ layer groups). The objective is to maximize accuracy on a 600-size validation dataset after merging:
\begin{equation}
    \boldsymbol{z}^*
    =
    \arg\max_{\boldsymbol{z}\in\mathbb{R}^{L\times 2N}}
    f\!\left(\boldsymbol{\theta}_*(\boldsymbol{z})\right),
    \label{eq:search_objective_main}
\end{equation}
where $f$ denotes accuracy on a validation set.

\paragraph{CMA-ES optimization.}
We optimize Eq.~\eqref{eq:search_objective_main} with CMA-ES. Following the DARE--TIES merging procedure, we first perform a grid search over shared merging hyperparameters to find a stable starting point. CMA-ES is then initialized from this point and directly searches the full vector $\boldsymbol{z}$, allowing each layer group to choose direction-specific drop rates and mixing weights. It maintains a Gaussian search distribution over the $L\times 2N$-dimensional coefficient space,
\begin{equation}
    \boldsymbol{z}_k^{(t)}
    \sim
    \mathcal{N}\!\left(
    \boldsymbol{z}^{(t)},
    (\sigma_{\boldsymbol{z}}^{(t)})^2 \boldsymbol{C}_{\boldsymbol{z}}^{(t)}
    \right),
    \quad k=1,\ldots,G.
    \label{eq:cma_sampling_main}
\end{equation}
Each sampled candidate is decoded into a merged model $\boldsymbol{\theta}_*(\boldsymbol{z}_k^{(t)})$ and evaluated by validation accuracy. After ranking candidates by their validation fitness $f\!\left(\boldsymbol{\theta}_*(\boldsymbol{z})\right)$, the top-$\mu$ (we have $\mu=\frac{G}{2}$ in our experiments) candidates update the search mean as follows:
\begin{equation}
\begin{aligned}
    &\qquad\boldsymbol{z}^{(t+1)}
    =
    \sum_{k=1}^{\mu} \boldsymbol{z}_{k:G}^{(t)},
\end{aligned}
    \label{eq:cma_mean_main}
\end{equation}
where $\boldsymbol{z}_{k:G}^{(t)}$ denotes the $k$-th ranked candidate in generation $t$. We use the final distribution mean as $\boldsymbol{z}^*$ and return the merged model $\boldsymbol{\theta}_*(\boldsymbol{z}^*)$. 

\section{Experiments}\label{experiment-section}

In this section, we implement the proposed \textsc{Hyper-ES} algorithm to reinforce the mathematical reasoning capabilities of three LLMs, including \textit{Qwen2.5-0.5B-Instruct}, \textit{Qwen2.5-1.5B-Instruct}, and \textit{DeepSeek-R1-Distill-Qwen-1.5B}.

\subsection{Implementation Details}

\paragraph{Training \& Testing Settings.} This paper adopts two training datasets to incorporate reasoning step-count information, i.e., GSM8K-Aug~\citep{deng2024explicit} and the DeepScaler dataset~\citep{2025deepscaler}. We involve four arithmetic reasoning benchmarks (GSM8K, SVAMP, MultiArith, and GSM-Hard) and two high-difficulty mathematics benchmarks (MATH-500 and AMC2023) for testing. Test answers are verified using Math-Verify \citep{Kydlicek_Math-Verify_Math_Verification} after standard answer normalization. For the Qwen2.5 backbones, the maximum generation length is capped at 1024 tokens during both training and evaluation. For the DeepSeek-R1-Distill-Qwen-1.5B model, the generation length is limited to 6144 tokens during training and evaluation.

In preparing each $\Delta \theta$ for Hyper-ES, we use the inherent number of reasoning hops in GSM8K-Aug and categories in DeepScaler to build subsets. Each LoRA direction is produced by only 7 GRPO steps with a batch size of 256, consuming 1.8k samples per direction. We then build $N=10$ directions for Qwen2.5-Instruct and $N=7$ directions for DeepSeek-R1-Distill, which limits the total training data and the gradient calculation of the proposed Hyper-ES to 17,920 and 12,544 samples, respectively. In contrast, the GRPO/CMA-ES baseline is trained on the full set of direction-generation subsets (utilizing approximately 20,000 samples for GSM8K-Aug and 14,000 samples for DeepScaler), which requires more training steps and samples in total. The specific LoRA configurations, including rank $r=32$ and scaling $\alpha=64$, are provided in Appendix~\ref{app:training_lora}. In implementing the baseline reinforcement learning training, we utilize the verl-0.4.x framework for RLVR. Please refer to Appendix~\ref{app:search_details} for more details in the ES stage.

\paragraph{Baselines.} We compare the proposed \textsc{Hyper-ES} against several baselines: (i) the \textbf{Base} model without fine-tuning, (ii) \textbf{GRPO+LoRA}~\citep{shao2024deepseekmath} trained on the full dataset with equivalent gradient update steps, (iii) \textbf{CMA-ES+LoRA} directly applied over the original LoRA parameter space, and (iv) \textbf{Average Merge}~\citep{ilharco2022editing}, which averages all LoRA task vectors uniformly. For all evaluated models, the Qwen2.5 arithmetic benchmarks are evaluated under greedy decoding (temperature $\tau=0$). For the DeepSeek-R1-Distill-Qwen-1.5B model, we use sampling decoding under standard evaluation configurations, setting the temperature $\tau=0.6$ and top-p to $0.95$.

\paragraph{Metrics.} For the Qwen2.5 LLMs, we evaluate using standard accuracy (\%) under greedy decoding ($T=0$). For the DeepSeek-R1-Distill-Qwen-1.5B, we report Mean@32—defined as the average Pass@1 accuracy over 32 independent sampling runs—to effectively reduce the variance of the generations on competitive mathematics datasets.

\subsection{Performance on Arithmetic Reasoning Benchmarks}

As shown in Table~\ref{tab:main}, we present the evaluation of \textsc{Hyper-ES} and baselines on arithmetic reasoning datasets. \textbf{DARE+CMA-ES} represents applying CMA-ES for coefficients over the vanilla DARE merging method \cite{yu2024dare}. Overall, \textsc{Hyper-ES} outperforms direct weight-space evolutionary search and matches or slightly exceeds the GRPO+LoRA baseline.

On the Qwen2.5-0.5B-Instruct base LLM, \textsc{Hyper-ES} achieves an average accuracy of 57.13\%, slightly outperforming the GRPO+LoRA baseline (56.23\%) and yielding a substantial margin over CMA-ES+LoRA (52.76\%). On Qwen2.5-1.5B-Instruct, \textsc{Hyper-ES} achieves an average score of 74.26\% (compared to 73.51\% for GRPO+LoRA). Compared to Average Merge, which yields 72.36\%, \textsc{Hyper-ES} demonstrates the advantage of optimizing layer-wise weighting coefficients instead of executing uniform parameter mixing.

\begin{table*}[t]
    \centering
    \renewcommand\arraystretch{1}
    \setlength{\tabcolsep}{3mm}
    \caption{Arithmetic reasoning results on Qwen2.5 models. Accuracy (\%) is reported. Best result per benchmark within each model group is \underline{underlined}; best average is \textbf{boldfaced}; second-best average is \colorbox[HTML]{E8E8E8}{shaded}.}
    \small
    \begin{tabular}{l|cccc|c}
        \toprule[0.5mm]
        \multicolumn{1}{c|}{}
        & \multicolumn{1}{c}{In-domain}
        & \multicolumn{3}{c|}{Out-of-domain}
        & \multicolumn{1}{c}{ } \\
        \cmidrule(lr){2-2}\cmidrule(lr){3-5}
        \multicolumn{1}{c|}{Method}
        & GSM8K & SVAMP & MultiArith & GSM-Hard & Avg. \\
        \midrule[0.3mm]
        \multicolumn{6}{c}{\textbf{Qwen2.5-0.5B-Instruct}} \\
        \midrule
        Base
        & 46.93 & 49.67 & 78.33 & 16.60 & 47.88 \\
        GRPO+LoRA
        & 50.57 & 61.33 & \underline{95.56} & 17.44 & \cellcolor[HTML]{E8E8E8}56.23 \\
        CMA-ES+LoRA
        & 45.94 & 56.33 & 92.22 & 16.53 & 52.76 \\
        Average Merge
        & \underline{51.48} & 61.00 & 92.78 & 17.29 & 55.64 \\
        DARE+CMA-ES & 47.69 & 53.33 & 87.22 & 16.15 & 51.10 \\
        \midrule
        \textbf{\textsc{Hyper-ES}{} (Ours)}
        & 51.18 & \underline{65.33} & 93.89 & \underline{18.12} & \textbf{57.13} \\
        \midrule[0.3mm]
        \multicolumn{6}{c}{\textbf{Qwen2.5-1.5B-Instruct}} \\
        \midrule
        Base
        & 71.49 & 82.00 & 93.33 & 35.94 & 70.69 \\
        GRPO+LoRA
        & 74.22 & \underline{84.67} & \underline{97.22} & 37.91 & \cellcolor[HTML]{E8E8E8}73.51 \\
        CMA-ES+LoRA
        & 73.54 & 82.67 & 95.00 & 36.92 & 72.03 \\
        Average Merge
        & 71.42 & 82.67 & 93.33 & 35.56 & 70.75 \\
        \midrule
        \textbf{\textsc{Hyper-ES}{} (Ours)}
        & \underline{75.51} & \underline{84.67} & \underline{97.22} & \underline{39.65} & \textbf{74.26} \\
        \bottomrule[0.5mm]
    \end{tabular}
    \label{tab:main}
\end{table*}
\begin{table*}[t]
    \centering
    \renewcommand\arraystretch{1}
    \setlength{\tabcolsep}{4.6mm}
    \caption{Results on DeepSeek-R1-Distill-Qwen-1.5B trained with the DeepScaler competition curriculum. All scores are Mean@32 (\%)---the average of 32 independent Pass@1 evaluations, multiplied by 100. Best result per benchmark is \underline{underlined}; best average is \textbf{boldfaced}; second-best average is \colorbox[HTML]{E8E8E8}{shaded}.}
    \small
    \begin{tabular}{l|ccc|c}
        \toprule[0.5mm]
        \multicolumn{5}{c}{\textbf{DeepSeek-R1-Distill-Qwen-1.5B}} \\
        \midrule[0.3mm]
        \multicolumn{1}{c|}{Method}
        & GSM8K & MATH-500 & AMC-23 & Avg. \\
        \midrule[0.3mm]
        Base
        & 76.06 & 74.79 & 57.34 & 69.40 \\
        GRPO+LoRA
        & 76.07 & 75.48 & \underline{59.61} & \cellcolor[HTML]{E8E8E8}70.39 \\
        CMA-ES+LoRA
        & 76.10 & 74.97 & 57.81 & 69.63 \\
        Average Merge
        & 75.90 & 75.29 & 57.11 & 69.43 \\
        \midrule
        \textbf{\textsc{Hyper-ES}{} (Ours)}
        & \underline{76.49} & \underline{76.97} & 59.45 & \textbf{70.97} \\
        \bottomrule[0.5mm]
    \end{tabular}
    \label{tab:deepseek}
\end{table*}

\begin{table*}[t]
    \centering\small
    \renewcommand\arraystretch{1}
    \setlength{\tabcolsep}{6pt}
    \caption{Ablation study on Qwen2.5-0.5B-Instruct. All scores are test-set accuracies (\%). \textbf{Avg.}\ = mean of GSM8K, SVAMP, MultiArith, and GSM-Hard. The $\Delta$ column reports average-score change.}
    \label{tab:ablation}
    \begin{tabular}{l c c c c c c}
        \toprule[0.5mm]
        \textbf{Variant} & \textbf{GSM8K} & \textbf{SVAMP} & \textbf{MultiArith} & \textbf{GSM-Hard} & \textbf{Avg.} & \textbf{Avg. $\Delta$} \\
        \midrule[0.3mm]
        \textbf{\textsc{Hyper-ES}{} (Ours)} & 51.18 & 65.33 & 93.89 &  18.12 & \textbf{57.13} & \textcolor{ForestGreen}{+0.00} \\
        \textit{w/o} CMA-ES & 50.72 & 60.67 & 95.00 & 17.29 & 55.92 & \textcolor{red}{-1.21} \\
        \midrule[0.3mm]
        \textit{w/o} Grouping & 48.75 & 57.33 & 91.11 & 17.29 & 53.62 & \textcolor{red}{-3.51} \\
        \bottomrule[0.5mm]
    \end{tabular}
\end{table*}

\begin{table*}[t]
    \centering
    \renewcommand\arraystretch{1.05}
    \setlength{\tabcolsep}{3pt}
    \caption{Efficiency-oriented ablations on Qwen2.5-0.5B-Instruct. All scores are test-set accuracies (\%).}
    \label{tab:efficiency_stress}
    \small
    \begin{tabular}{l c c c c c c}
        \toprule[0.5mm]
        \textbf{Variant} & \textbf{Gradient Samples} & \textbf{GSM8K} & \textbf{SVAMP} & \textbf{MultiArith} & \textbf{GSM-Hard} & \textbf{Avg.} \\
        \midrule[0.3mm]
        GRPO+LoRA & 20{,}000 & 50.57 & 61.33 & \textbf{95.56} & 17.44 & 56.23 \\ \midrule
        \textbf{\textsc{Hyper-ES} - (7 gradients each direction)} & 17{,}920 & \textbf{51.18} & \textbf{65.33} & 93.89 & \textbf{18.12} & \textbf{57.13} \\
        \textbf{\textsc{Hyper-ES}} - (4 gradients each direction) & 10{,}240 & 49.81 & 60.67 & 95.00 & 17.36 & 55.71 \\ \midrule
        \textbf{\textsc{Hyper-ES}} - Pure-random directions $\Delta\theta$ & 0 & 47.01 & 48.33 & 77.78 & 17.21 & 47.58 \\
        \textbf{\textsc{Hyper-ES}} - CMA-ES directions $\Delta\theta$ & 0 & 45.19 & 57.00 & 87.22 & 14.48 & 50.97 \\
        \bottomrule[0.5mm]
    \end{tabular}
\end{table*}

\subsection{Performance on Mathematical Reasoning Capabilities}

To evaluate the generalizability of \textsc{Hyper-ES}, we conduct experiments on DeepSeek-R1-Distill-Qwen-1.5B across three mathematical reasoning benchmarks, including GSM8K, MATH-500, and AMC2023. As shown in Table~\ref{tab:deepseek}, \textsc{Hyper-ES} reaches an average score of 70.97\%, demonstrating performance advantages over both GRPO+LoRA (70.39\%) and CMA-ES+LoRA (69.63\%).

Specifically, \textsc{Hyper-ES} provides measurable gains on MATH-500, increasing the score to 76.97\% (compared to 75.48\% for GRPO+LoRA), and GSM8K, increasing the score to 76.49\% (compared to 76.07\% for GRPO+LoRA). On AMC2023, \textsc{Hyper-ES} remains competitive with GRPO+LoRA. These results demonstrate that searching over combinations of task-relevant descent directions generalizes well across diverse mathematical reasoning tasks. Synthesizing the results on 3 LLMs and 6 benchmarks, \textsc{Hyper-ES} is the most effective resource-constrained fine-tuning method for LLM Reasoning. More results, including multi-seed evaluation, additional baselines, code-generation evaluation, and detailed execution time, are provided in Appendix~\ref{app:additional_results}.

\section{Discussion and Analysis}
\label{sec:analysis}

\subsection{Ablation Study}\label{sec:ablation_setup}

To evaluate the contribution of each component, we perform ablations on Qwen2.5-0.5B-Instruct. Table~\ref{tab:ablation} uses the full \textsc{Hyper-ES}{} results as references.

\paragraph{\textit{w/o} CMA-ES.}
Replacing CMA-ES with the grid-only variant drops the average from 57.13\% to 55.92\% ($-1.21$). This shows that the coarse grid search can identify a useful global drop-rate and scale, but the layer-wise CMA-ES search still contributes by adapting the merge coefficients to different parameter groups.


\paragraph{\textit{w/o} Grouping.}
Replacing difficulty-aware grouping with random grouping causes the largest drop among these ablations, from 57.13\% to 53.62\% ($-3.51$). This supports the role of structured direction-pool construction.

\subsection{Discussion on Efficiency}
\label{sec:efficiency}

\paragraph{Memory and wall-clock efficiency.}
The main extra cost of \textsc{Hyper-ES} is constructing several short GRPO-induced directions. Since these runs are independent, the peak GPU memory is comparable to a single GRPO+LoRA run and does not grow with the number of directions. The later merging and CMA-ES stages use stored task vectors and validation evaluations, avoiding gradient-memory overhead. As shown in Appendix~\ref{sec:time_efficiency}, the direction runs can also be parallelized across GPU nodes without communication, reducing wall-clock time by 1 hour on Qwen-0.5B.

\paragraph{Training-sample budget.}
Table~\ref{tab:efficiency_stress} shows that fewer gradient steps are sufficient to construct useful directions. The full setting uses 7 updates per direction, totaling 17{,}920 samples, and reaches 57.13\% average accuracy. Reducing this to 4 updates uses only 10{,}240 samples while still achieving 55.71\%. Thus, \textsc{Hyper-ES} can also save time without parallelism by shortening direction construction.

\paragraph{Task-aligned directions are necessary.}
The gains are not from low-dimensional search alone. Random directions drop to 47.58\%, close to the base model, and CMA-ES without task-aligned directions also fails to match \textsc{Hyper-ES}. This supports our claim that the key is searching over meaningful GRPO-induced descent directions, not arbitrary subspaces.

\subsection{The Effectiveness of Hyper-ES Mainly due to the CMA-ES Stage}

\textsc{Hyper-ES} is effective because it searches over combinations of task-aligned directions. The complete GSM8K results in Appendix~\ref{app:full_eval} show that individual short LoRA directions are weak, but merging them is beneficial: on Qwen2.5-0.5B, grid-only \textsc{Hyper-ES} improves over GRPO+LoRA from 50.57\% to 50.72\%, and CMA-ES further improves to 51.18\%. On Qwen2.5-1.5B, the same comparison is 74.22\% $\rightarrow$ 75.21\% $\rightarrow$ 75.51\%.

This suggests that the gains mainly come from the search stage. The grid search finds a useful global merge configuration, while CMA-ES provides additional layer-wise coefficient adaptation.

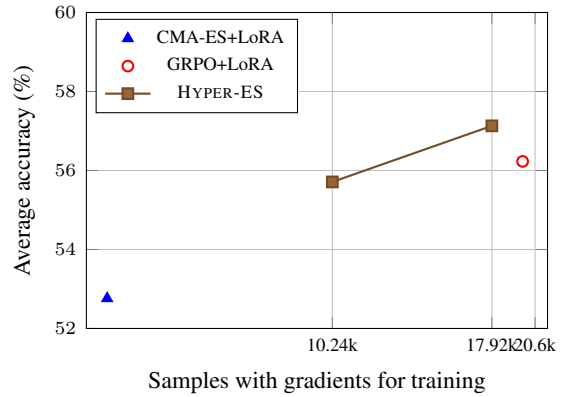
\begin{figure}[t]
    \centering
    \begin{tikzpicture}
        \begin{axis}[
            width=0.48\textwidth,
            height=0.36\textwidth,
            xlabel={Samples with gradients for training},
            ylabel={Average accuracy (\%)},
            xmin=-1000, xmax=21200,
            ymin=52, ymax=60,
            grid=both,
            legend style={at={(0.02,0.98)},anchor=north west,font=\scriptsize},
            tick label style={font=\scriptsize},
            label style={font=\small},
            scaled x ticks=false,
            xtick={10840,18520,20600},
            xticklabels={10.24k,17.92k,20.6k},
        ]
            \addplot+[only marks, mark=triangle*, thick] coordinates {(0,52.76)};
            \addlegendentry{CMA-ES+LoRA}
            \addplot+[only marks, mark=o, thick] coordinates {(20000,56.23)};
            \addlegendentry{GRPO+LoRA}
            \addplot+[mark=square*, thick] coordinates {(10840,55.71) (18520,57.13)};
            \addlegendentry{\textsc{Hyper-ES}}
        \end{axis}
    \end{tikzpicture}
    \caption{Sample-efficiency curve on Qwen2.5-0.5B-Instruct arithmetic benchmarks. \textsc{Hyper-ES} achieves higher average accuracy than GRPO+LoRA with fewer training samples, while CMA-ES+LoRA remains lower.}
    \label{fig:grpo_efficiency_curve}
\end{figure}

\section{Conclusion}

This paper presents \textsc{Hyper-ES}, an ES framework that searches over descent direction merging hyperparameters instead of directly perturbing LLM weights. The method addresses the high-dimensional random-walk failure mode of direct ES by constructing a compact, task-aligned search space from brief direction-generation updates. Empirically, \textsc{Hyper-ES} improves over average merging and grid-only merging and slightly improves over single-stage GRPO on the reported Qwen2.5 arithmetic benchmarks and DeepSeek-R1-Distill-Qwen-1.5B hard math benchmarks. The results suggest that short gradient-based runs can be used not only to train models directly, but also to define useful low-dimensional spaces for later derivative-free ES optimization.
\section*{Limitations}
\label{sec:limitations}

This paper only considers GRPO as an effective method for fine-tuning large models for reasoning. Considering OPD-based methods and SFT on certain specific datasets are also suitable as initial gradients. In the future, we will consider further demonstrating whether the effectiveness of Hyper-ES is limited to GRPO. 

We will also consider online dynamic methods as future work, obtaining the best direction through model iteration and performing ES search.


\bibliography{custom}

@article{ilharco2022editing,
  title={Editing models with task arithmetic},
  author={Ilharco, Gabriel and Ribeiro, Marco Tulio and Wortsman, Mitchell and Gururangan, Suchin and Schmidt, Ludwig and Hajishirzi, Hannaneh and Farhadi, Ali},
  journal={arXiv preprint arXiv:2212.04089},
  year={2022}
}

@article{zheng2026beyond,
  title={Beyond Imitation: Reinforcement Learning for Active Latent Planning},
  author={Zheng, Zhi and Lee, Wee Sun},
  journal={arXiv preprint arXiv:2601.21598},
  year={2026}
}

@article{zheng2025soft,
  title={Soft-grpo: Surpassing discrete-token llm reinforcement learning via gumbel-reparameterized soft-thinking policy optimization},
  author={Zheng, Zhi and Gu, Yu and Liu, Wei and Teh, Yee Whye and Lee, Wee Sun},
  journal={arXiv preprint arXiv:2511.06411},
  year={2025}
}

@article{yadav2023tiesmerging,
  title={Ties-merging: Resolving interference when merging models},
  author={Yadav, Prateek and Tam, Derek and Choshen, Leshem and Raffel, Colin A and Bansal, Mohit},
  journal={Advances in neural information processing systems},
  volume={36},
  pages={7093--7115},
  year={2023}
}

@article{liu2025understanding,
  title={Understanding r1-zero-like training: A critical perspective},
  author={Liu, Zichen and Chen, Changyu and Li, Wenjun and Qi, Penghui and Pang, Tianyu and Du, Chao and Lee, Wee Sun and Lin, Min},
  journal={arXiv preprint arXiv:2503.20783},
  year={2025}
}

@article{liu2025ea4llm,
  title={EA4LLM: A Gradient-Free Approach to Large Language Model Optimization via Evolutionary Algorithms},
  author={Liu, WenTao and Song, Siyu and Hao, Hao and Zhou, Aimin},
  journal={arXiv preprint arXiv:2510.10603},
  year={2025}
}

@article{gan2026neural,
  title={Neural thickets: Diverse task experts are dense around pretrained weights},
  author={Gan, Yulu and Isola, Phillip},
  journal={arXiv preprint arXiv:2603.12228},
  year={2026}
}

@article{fu2026reasoning,
  title={Reasoning Resides in Layers: Restoring Temporal Reasoning in Video-Language Models with Layer-Selective Merging},
  author={Fu, Zihang and Wang, Haonan and Kang, Jian and Kawaguchi, Kenji and Wu, Jiaying},
  journal={arXiv preprint arXiv:2604.11399},
  year={2026}
}

@article{li2025system,
  title={From system 1 to system 2: A survey of reasoning large language models},
  author={Li, Zhong-Zhi and Zhang, Duzhen and Zhang, Ming-Liang and Zhang, Jiaxin and Liu, Zengyan and Yao, Yuxuan and Xu, Haotian and Zheng, Junhao and Wang, Pei-Jie and Chen, Xiuyi and others},
  journal={arXiv preprint arXiv:2502.17419},
  year={2025}
}

@inproceedings{yu2024dare,
  title={Language models are super mario: Absorbing abilities from homologous models as a free lunch},
  author={Yu, Le and Yu, Bowen and Yu, Haiyang and Huang, Fei and Li, Yongbin},
  booktitle={Forty-first International Conference on Machine Learning},
  year={2024}
}

@article{akiba2025evolutionary,
  title={Evolutionary optimization of model merging recipes},
  author={Akiba, Takuya and Shing, Makoto and Tang, Yujin and Sun, Qi and Ha, David},
  journal={Nature Machine Intelligence},
  volume={7},
  number={2},
  pages={195--204},
  year={2025},
  publisher={Nature Publishing Group UK London}
}

@article{qiu2025evolution,
  title={Evolution strategies at scale: Llm fine-tuning beyond reinforcement learning},
  author={Qiu, Xin and Gan, Yulu and Hayes, Conor F and Liang, Qiyao and Xu, Yinggan and Dailey, Roberto and Meyerson, Elliot and Hodjat, Babak and Miikkulainen, Risto},
  journal={arXiv preprint arXiv:2509.24372},
  year={2025}
}

@article{shao2024deepseekmath,
  title={Deepseekmath: Pushing the limits of mathematical reasoning in open language models},
  author={Shao, Zhihong and Wang, Peiyi and Zhu, Qihao and Xu, Runxin and Song, Junxiao and Bi, Xiao and Zhang, Haowei and Zhang, Mingchuan and Li, YK and Wu, Yang and others},
  journal={arXiv preprint arXiv:2402.03300},
  year={2024}
}

@article{cobbe2021gsm8k,
  title={Training verifiers to solve math word problems},
  author={Cobbe, Karl and Kosaraju, Vineet and Bavarian, Mohammad and Chen, Mark and Jun, Heewoo and Kaiser, Lukasz and Plappert, Matthias and Tworek, Jerry and Hilton, Jacob and Nakano, Reiichiro and others},
  journal={arXiv preprint arXiv:2110.14168},
  year={2021}
}

@misc{qwen2025qwen25,
      title={Qwen2.5 Technical Report}, 
      author={Qwen and : and An Yang and Baosong Yang and Beichen Zhang and Binyuan Hui and Bo Zheng and Bowen Yu and Chengyuan Li and Dayiheng Liu and Fei Huang and Haoran Wei and Huan Lin and Jian Yang and Jianhong Tu and Jianwei Zhang and Jianxin Yang and Jiaxi Yang and Jingren Zhou and Junyang Lin and Kai Dang and Keming Lu and Keqin Bao and Kexin Yang and Le Yu and Mei Li and Mingfeng Xue and Pei Zhang and Qin Zhu and Rui Men and Runji Lin and Tianhao Li and Tianyi Tang and Tingyu Xia and Xingzhang Ren and Xuancheng Ren and Yang Fan and Yang Su and Yichang Zhang and Yu Wan and Yuqiong Liu and Zeyu Cui and Zhenru Zhang and Zihan Qiu},
      year={2025},
      eprint={2412.15115},
      archivePrefix={arXiv},
      primaryClass={cs.CL},
      url={https://arxiv.org/abs/2412.15115}, 
}

@misc{Kydlicek_Math-Verify_Math_Verification,
author = {Kydlíček, Hynek},
license = {Apache-2.0},
title = {{Math-Verify: Math Verification Library}},
url = {https://github.com/huggingface/math-verify},
version = {0.6.1},
  year={2025}
}

@article{hoy2026matching,
  title={Matching Accuracy, Different Geometry: Evolution Strategies vs GRPO in LLM Post-Training},
  author={Hoy, William and Wang, Binxu and Pan, Xu},
  journal={arXiv preprint arXiv:2604.01499},
  year={2026}
}

@article{guo2025deepseek,
  title={Deepseek-r1: Incentivizing reasoning capability in llms via reinforcement learning},
  author={Guo, Daya and Yang, Dejian and Zhang, Haowei and Song, Junxiao and Wang, Peiyi and Zhu, Qihao and Xu, Runxin and Zhang, Ruoyu and Ma, Shirong and Bi, Xiao and others},
  journal={arXiv preprint arXiv:2501.12948},
  year={2025}
}

@article{2025deepscaler,
  title={Deepscaler: Surpassing o1-preview with a 1.5 b model by scaling rl},
  author={Luo, Michael and Tan, Sijun and Wong, Justin and Shi, Xiaoxiang and Tang, William Y and Roongta, Manan and Cai, Colin and Luo, Jeffrey and Zhang, Tianjun and Li, Li Erran and others},
  journal={Notion Blog},
  volume={3},
  number={5},
  year={2025}
}

@article{liu2025part,
  title={Part i: Tricks or traps? a deep dive into rl for llm reasoning},
  author={Liu, Zihe and Liu, Jiashun and He, Yancheng and Wang, Weixun and Liu, Jiaheng and Pan, Ling and Hu, Xinyu and Xiong, Shaopan and Huang, Ju and Hu, Jian and others},
  journal={arXiv preprint arXiv:2508.08221},
  year={2025}
}

@article{yu2025dapo,
  title={Dapo: An open-source llm reinforcement learning system at scale, 2025},
  author={Yu, Qiying and Zhang, Zheng and Zhu, Ruofei and Yuan, Yufeng and Zuo, Xiaochen and Yue, Yu and Dai, Weinan and Fan, Tiantian and Liu, Gaohong and Liu, Lingjun and others},
  journal={URL https://arxiv. org/abs/2503.14476},
  volume={1},
  pages={2},
  year={2025}
}

@article{tajwar2026maximum,
  title={Maximum Likelihood Reinforcement Learning},
  author={Tajwar, Fahim and Zeng, Guanning and Zhou, Yueer and Song, Yuda and Arora, Daman and Jiang, Yiding and Schneider, Jeff and Salakhutdinov, Ruslan and Feng, Haiwen and Zanette, Andrea},
  journal={arXiv preprint arXiv:2602.02710},
  year={2026}
}

@article{yin2025evaluating,
  title={Evaluating Parameter Efficient Methods for RLVR},
  author={Yin, Qingyu and Wu, Yulun and Shen, Zhennan and Li, Sunbowen and Wang, Zhilin and Li, Yanshu and Leong, Chak Tou and Kang, Jiale and Gu, Jinjin},
  journal={arXiv preprint arXiv:2512.23165},
  year={2025}
}

@article{chen2025bring,
  title={Bring reason to vision: Understanding perception and reasoning through model merging},
  author={Chen, Shiqi and Zhang, Jinghan and Zhu, Tongyao and Liu, Wei and Gao, Siyang and Xiong, Miao and Li, Manling and He, Junxian},
  journal={arXiv preprint arXiv:2505.05464},
  year={2025}
}

@article{yin2025towards,
  title={Towards Efficient Multimodal Unified Reasoning Model via Model Merging},
  author={Yin, Qixiang and Yao, Huanjin and Chen, Jianghao and Huang, Jiaxing and Zhao, Zhicheng and Su, Fei},
  journal={arXiv preprint arXiv:2510.08987},
  year={2025}
}

@article{abdi2026evolutionary,
  title={Evolutionary Strategies lead to Catastrophic Forgetting in LLMs},
  author={Abdi, Immanuel and Gupta, Akshat and Mok, Micah and Lu, Alexander and Lee, Nicholas and Anumanchipalli, Gopala},
  journal={arXiv preprint arXiv:2601.20861},
  year={2026}
}

@article{salimans2017evolution,
  title={Evolution strategies as a scalable alternative to reinforcement learning},
  author={Salimans, Tim and Ho, Jonathan and Chen, Xi and Sidor, Szymon and Sutskever, Ilya},
  journal={arXiv preprint arXiv:1703.03864},
  year={2017}
}

@article{sun2026essam,
  title={ESSAM: A Novel Competitive Evolution Strategies Approach to Reinforcement Learning for Memory Efficient LLMs Fine-Tuning},
  author={Sun, Zhishen and Dang, Sizhe and Dai, Guang and Ye, Haishan},
  journal={arXiv preprint arXiv:2602.01003},
  year={2026}
}

@article{sarkar2025evolution,
  title={Evolution strategies at the hyperscale},
  author={Sarkar, Bidipta and Fellows, Mattie and Duque, Juan Agustin and Letcher, Alistair and Villares, Antonio Le{\'o}n and Sims, Anya and Wibault, Clarisse and Samsonov, Dmitry and Cope, Dylan and Liesen, Jarek and others},
  journal={arXiv preprint arXiv:2511.16652},
  year={2025}
}

@article{wei2022chain,
  title={Chain-of-thought prompting elicits reasoning in large language models},
  author={Wei, Jason and Wang, Xuezhi and Schuurmans, Dale and Bosma, Maarten and Xia, Fei and Chi, Ed and Le, Quoc V and Zhou, Denny and others},
  journal={Advances in neural information processing systems},
  volume={35},
  pages={24824--24837},
  year={2022}
}

@article{liu2025efficient,
  title={Efficient inference for large reasoning models: A survey},
  author={Liu, Yue and Wu, Jiaying and He, Yufei and Gong, Ruihan and Xia, Jun and Li, Liang and Gao, Hongcheng and Chen, Hongyu and Bi, Baolong and Zhang, Jiaheng and others},
  journal={arXiv preprint arXiv:2503.23077},
  year={2025}
}

@article{zheng2025reasoning,
  title={Reasoning-CV: Fine-tuning Powerful Reasoning LLMs for Knowledge-Assisted Claim Verification},
  author={Zheng, Zhi and Lee, Wee Sun},
  journal={arXiv preprint arXiv:2505.12348},
  year={2025}
}

@article{park2025mobilerag,
  title={MobileRAG: A Fast, Memory-Efficient, and Energy-Efficient Method for On-Device RAG},
  author={Park, Taehwan and Lee, Geonho and Kim, Min-Soo},
  journal={arXiv preprint arXiv:2507.01079},
  year={2025}
}

@inproceedings{lightman2024verify,
  title={Let's verify step by step},
  author={Lightman, Hunter and Kosaraju, Vineet and Burda, Yuri and Edwards, Harrison and Baker, Bowen and Lee, Teddy and Leike, Jan and Schulman, John and Sutskever, Ilya and Cobbe, Karl},
  booktitle={International Conference on Learning Representations},
  volume={2024},
  pages={39578--39601},
  year={2024}
}

@inproceedings{patel2021nlp,
  title={Are NLP models really able to solve simple math word problems?},
  author={Patel, Arkil and Bhattamishra, Satwik and Goyal, Navin},
  booktitle={Proceedings of the 2021 conference of the North American chapter of the association for computational linguistics: human language technologies},
  pages={2080--2094},
  year={2021}
}

@inproceedings{roy2015solving,
  title={Solving general arithmetic word problems},
  author={Roy, Subhro and Roth, Dan},
  booktitle={Proceedings of the 2015 conference on empirical methods in natural language processing},
  pages={1743--1752},
  year={2015}
}

@inproceedings{gao2023pal,
  title={Pal: Program-aided language models},
  author={Gao, Luyu and Madaan, Aman and Zhou, Shuyan and Alon, Uri and Liu, Pengfei and Yang, Yiming and Callan, Jamie and Neubig, Graham},
  booktitle={International conference on machine learning},
  pages={10764--10799},
  year={2023},
  organization={PMLR}
}

@article{hansen2001cmaes,
  title={Completely derandomized self-adaptation in evolution strategies},
  author={Hansen, Nikolaus and Ostermeier, Andreas},
  journal={Evolutionary computation},
  volume={9},
  number={2},
  pages={159--195},
  year={2001},
  publisher={MIT Press}
}

@article{ma2025ero,
  title={Evolutionary System 2 Reasoning: An Empirical Proof},
  author={Ma, Zeyuan and Huang, Wenqi and Song, Guo-Huan and Guo, Hongshu and Ma, Sijie and Cao, Zhiguang and Gong, Yue-Jiao},
  journal={arXiv preprint arXiv:2512.05760},
  year={2025}
}

@article{korotyshova2025essa,
  title={ESSA: Evolutionary Strategies for Scalable Alignment},
  author={Korotyshova, Daria and Shaposhnikov, Boris and Malakhov, Alexey and Khokhulin, Alexey and Surnachev, Nikita and Ovcharenko, Kirill and Bredis, George and Gorbatovski, Alexey and Sinii, Viacheslav and Gavrilov, Daniil},
  journal={arXiv preprint arXiv:2507.04453},
  year={2025}
}

@article{yang2026model,
  title={Model merging in llms, mllms, and beyond: Methods, theories, applications, and opportunities},
  author={Yang, Enneng and Shen, Li and Guo, Guibing and Wang, Xingwei and Cao, Xiaochun and Zhang, Jie and Tao, Dacheng},
  journal={ACM Computing Surveys},
  volume={58},
  number={8},
  pages={1--41},
  year={2026},
  publisher={ACM New York, NY}
}

@article{huang2023lorahub,
  title={Lorahub: Efficient cross-task generalization via dynamic lora composition},
  author={Huang, Chengsong and Liu, Qian and Lin, Bill Yuchen and Pang, Tianyu and Du, Chao and Lin, Min},
  journal={arXiv preprint arXiv:2307.13269},
  year={2023}
}

@article{deng2024explicit,
  title={From explicit cot to implicit cot: Learning to internalize cot step by step},
  author={Deng, Yuntian and Choi, Yejin and Shieber, Stuart},
  journal={arXiv preprint arXiv:2405.14838},
  year={2024}
}

@article{yang2025cabs,
  title={Cabs: Conflict-aware and balanced sparsification for enhancing model merging},
  author={Yang, Zongzhen and Qi, Binhang and Sun, Hailong and Long, Wenrui and Zhao, Ruobing and Gao, Xiang},
  journal={arXiv preprint arXiv:2503.01874},
  year={2025}
}

@article{cheng2025whoever,
  title={Whoever started the interference should end it: Guiding data-free model merging via task vectors},
  author={Cheng, Runxi and Xiong, Feng and Wei, Yongxian and Zhu, Wanyun and Yuan, Chun},
  journal={arXiv preprint arXiv:2503.08099},
  year={2025}
}

@article{austin2021program,
  title={Program synthesis with large language models},
  author={Austin, Jacob and Odena, Augustus and Nye, Maxwell and Bosma, Maarten and Michalewski, Henryk and Dohan, David and Jiang, Ellen and Cai, Carrie and Terry, Michael and Le, Quoc and others},
  journal={arXiv preprint arXiv:2108.07732},
  year={2021}
}

@misc{openthought2025tinygrpo,
  author={{Open-Thought}},
  license={Apache-2.0},
  title={{tiny-grpo: Minimal Hackable GRPO Implementation}},
  url={https://github.com/open-thought/tiny-grpo},
  year={2025}
}

@article{chen2021evaluating,
  title={Evaluating large language models trained on code},
  author={Chen, Mark and Tworek, Jerry and Jun, Heewoo and Yuan, Qiming and Pinto, Henrique Ponde De Oliveira and Kaplan, Jared and Edwards, Harri and Burda, Yuri and Joseph, Nicholas and Brockman, Greg and others},
  journal={arXiv preprint arXiv:2107.03374},
  year={2021}
}

\appendix
\newpage
\clearpage
\section{Related Work}
\label{sec:related_work}

\paragraph{Evolution Strategies for LLM Fine-Tuning.}
As advanced LLMs are equipped with larger and larger parameters, fine-tuning them in devices with moderate scales becomes unaffordable, and ES have emerged as efficient gradient-free optimizers for LLMs. \citet{qiu2025evolution} demonstrated that ES can scale to full-parameter LLM fine-tuning, matching or exceeding GRPO in sample efficiency and training stability. This momentum has driven ES-based methods into pre-training \cite{liu2025ea4llm}, System~2 reasoning~\citep{ma2025ero}, few-shot adaptation \cite{gan2026neural,korotyshova2025essa}, and memory-efficient tuning via sharpness-aware mechanisms~\citep{sun2026essam}. However, directly applying ES to the full LLM parameter space induces an isotropic random walk in the high-dimensional weight space, simultaneously slowing convergence and causing catastrophic forgetting of the pre-trained manifold~\citep{abdi2026evolutionary}. \textsc{Hyper-ES} resolves this fundamental limitation by restricting CMA-ES to the low-dimensional manifold of GRPO-derived task-vector merging coefficients, replacing unconstrained weight perturbations with task-aligned descent directions.

\paragraph{Reinforcement Learning for LLM Reasoning.}
GRPO~\citep{shao2024deepseekmath,guo2025deepseek} and its successors \cite{liu2025understanding,yu2025dapo,liu2025part,tajwar2026maximum,zheng2025soft,zheng2026beyond} have become the standard RLVR approach for reasoning, leveraging group-relative verifiable rewards to discover effective descent directions in the loss landscape. The standard RLVR pipeline, however, requires simultaneously maintaining actor, reference, and reward models in GPU memory, creating prohibitive overhead for resource-constrained deployments. \textsc{Hyper-ES} reframes GRPO as a \emph{direction provider}: short GRPO runs expose task-specific descent directions as LoRA task vectors, which are reused as the search basis for a subsequent memory-light CMA-ES stage. As shown in \citet{yin2025evaluating}, GRPO+LoRA can still demonstrate acceptable performance.

\paragraph{Model Merging in LLM.}
Model merging composes capabilities by linearly manipulating stored task vectors~\citep{yang2026model,chen2025bring,yin2025towards}. Task Arithmetic~\citep{ilharco2022editing} establishes the additive framework; DARE~\citep{yu2024dare} reduces interference via random sparsification, and TIES-Merging~\citep{yadav2023tiesmerging} via sign consensus. To further mitigate conflicts, CABS~\citep{yang2025cabs} introduces conflict-aware and balanced sparsification to eliminate parameter overlap, while WUDI-Merging~\citep{cheng2025whoever} guides model merging by modeling task vectors as approximate linear subspaces of linear layer inputs to eliminate interference. Some previous work applied evolutionary black-box search to post-trained weight merging (\citet{akiba2025evolutionary,fu2026reasoning,huang2023lorahub}) and used CMA-ES to optimize layer-wise DARE/TIES parameters to fuse models fine-tuned for different specializations of the same base, thus merging LoRA adapters across multi-task tasks. Unlike these approaches, Hyper-ES considers providing different feasible coarse directions for the same task and model merging is used to implement the ES search rather than to integrate multiple capabilities. Furthermore, Hyper-ES pioneers a paradigm shift: instead of merging different pre-trained specialized models, it reuses DARE-TIES parameterization as a generalized dimensionality-reduced manifold with arbitrary parameter increments ($\Delta W$). This elevates model merging from a post-trained merging tool to a core optimization mechanism—directly replacing full-space reinforcement learning updates and fundamentally solving the high-dimensional ES random walk problem.

\section{Prompt}
\label{app:prompts}

We use the same prompt format during rollout generation and evaluation.

\begin{dialogbox}{Prompt for Mathematical Reasoning.}
    \textcolor{ForestGreen}{\textbf{system}}

    Let's think step by step and output the final answer after
    ``\texttt{\#\#\#\#}''.

    \textcolor{RoyalBlue}{\textbf{user}}

    \textcolor{RoyalBlue}{\{Question\}}
\end{dialogbox}

\paragraph{Reward Function}
Rewards are binary: 1.0 for a correct answer and 0.0 otherwise. Answer extraction follows a priority order: (1) \texttt{\#\#\#\# <number>}; (2) \verb|\boxed{...}|; (3) natural-language patterns such as ``the answer is $X$''; and (4) the last number in the response as a fallback. Answers are normalized by stripping commas and casting whole-valued floats to integers before comparison.


\newpage

\section{Detailed Proof}
\label{app:theory}

\subsection{Proofs for the ES Random-Walk Lemmas}
\label{app:random_walk_proofs}

\begin{proof}[Proof of Lemma~\ref{lem:near_orthogonal}]
    By rotational symmetry of the unit sphere, we may set $g^*/\|g^*\|=e_1$ without loss of generality. For a uniformly sampled unit vector $u\in\mathbb{R}^d$, standard concentration on the sphere gives
    \begin{equation}
        \Pr\left(|\langle u,e_1\rangle|\geq \tau\right)
        \leq
        2\exp\!\left(-\frac{d\tau^2}{2}\right).
    \end{equation}
    Applying a union bound over $G$ independently sampled perturbation directions $u_1,\ldots,u_G$ yields
    \begin{equation}
        \Pr\!\left(
        \max_{1\leq k\leq G}|\langle u_k,e_1\rangle|
        \geq \tau
        \right)
        \leq
        2G\exp\!\left(-\frac{d\tau^2}{2}\right).
    \end{equation}
    For any fixed $\tau>0$, if $\log G=o(d)$, then $2G\exp(-d\tau^2/2)\to 0$. Hence
    \begin{equation}
        \max_{1\leq k\leq G}|\langle u_k,e_1\rangle|
        \xrightarrow[d\to\infty]{p}
        0.
    \end{equation}
    Since $\alpha_k=\arccos(|\langle u_k,e_1\rangle|)$ and $\arccos(\cdot)$ is continuous at $0$, this is equivalent to $\min_k\alpha_k\xrightarrow{p}\pi/2$.
\end{proof}

\begin{proof}[Proof of Lemma~\ref{lem:random_walk}]
    Write the ES update as the sum of a useful component parallel to $g^*$ and an orthogonal component $r_t$. When the useful component is negligible, the parameter displacement after $T$ steps is dominated by
    \begin{equation}
        R_T=\sum_{t=0}^{T-1}\eta_t r_t.
    \end{equation}
    Assuming the orthogonal components have mean-zero cross terms across iterations, $\mathbb{E}\langle r_t,r_s\rangle\approx 0$ for $t\neq s$, we obtain
    \begin{equation}
        \begin{aligned}
            \mathbb{E}\|R_T\|_2^2
            &=
            \mathbb{E}
            \left\|
            \sum_{t=0}^{T-1}\eta_t r_t
            \right\|_2^2 \\
            &=
            \sum_{t=0}^{T-1}\eta_t^2\mathbb{E}\|r_t\|_2^2
            +2\sum_{s<t}\eta_s\eta_t\mathbb{E}\langle r_s,r_t\rangle \\
            &\approx
            \rho^2\sum_{t=0}^{T-1}\eta_t^2.
        \end{aligned}
    \end{equation}
    Thus the orthogonal updates do not cancel in squared norm; they accumulate as random-walk drift in parameter space.
\end{proof}

\begin{proposition}[Norm bound for merged updates]
    Let $\overline{\Delta\theta}^{(\ell)}$ denote the merged update for group $\ell$. If DARE sparsification is applied without rescaling and the TIES disjoint merge uses non-negative mixing weights, then
    \begin{equation}
        \left\|\gamma\overline{\Delta\theta}^{(\ell)}\right\|_F
        \leq
        \gamma\max_i\left\|\Delta\theta_i^{(\ell)}\right\|_F.
    \end{equation}
\end{proposition}

\begin{proof}[Proof sketch]
    Sparsification cannot increase the Frobenius norm. The TIES disjoint merge computes, coordinate-wise, a weighted average over sign-agreeing entries with non-negative weights. Thus the unscaled merged update is bounded by the largest sparsified task-vector norm. Multiplication by $\gamma$ gives the result.
\end{proof}

\begin{proposition}[Alignment in the task-vector subspace]
    Let $\mathcal{S}=\mathrm{span}\{\Delta\theta_1,\ldots,\Delta\theta_N\}$ and let $g^*$ be a target descent direction. If the projection of $g^*$ onto $\mathcal{S}$ has cosine $c>0$, then an isotropic perturbation in $\mathcal{S}$ has expected alignment $\Theta(c/\sqrt{\dim\mathcal{S}})$ with $g^*$, whereas an isotropic perturbation in the full space has expected alignment $O(1/\sqrt{d})$.
\end{proposition}

This statement formalizes why the low-dimensional subspace alone is insufficient: it must also be task-aligned. GRPO-derived task vectors are intended to provide this alignment.

\begin{table}[H]
    \centering
    \caption{Pairwise cosine similarity between GRPO+LoRA and CMA-ES+LoRA task vectors on Qwen2.5-0.5B, each trained for 10 steps on GSM8K. The near-zero off-diagonal value supports the directional-mismatch evidence in Figure~\ref{fig:figure15}.}
    \label{tab:cosine_sim}
    \small
    \begin{tabular}{l c c}
        \toprule
        & \textbf{GRPO+LoRA} & \textbf{CMA-ES+LoRA} \\
        \midrule
        \textbf{GRPO+LoRA}      & $1.000$ & $8.59 \times 10^{-5}$ \\
        \textbf{CMA-ES+LoRA}    & $8.59 \times 10^{-5}$ & $1.000$ \\
        \bottomrule
    \end{tabular}
\end{table}

\begin{table}[H]
    \centering
    \caption{Total adapter $\ell_2$ norm $\|\Delta W\|_F$ summed over all expanded LoRA modules (scaling $\alpha/r=2.0$). CMA-ES+LoRA produces adapters significantly larger than GRPO+LoRA and \textsc{Hyper-ES}.}
    \label{tab:frobenius}
    \small
    \resizebox{0.5\textwidth}{!}{
    \begin{tabular}{l c c c}
        \toprule
        \textbf{Model} & \textbf{GRPO+LoRA} & \textbf{CMA-ES+LoRA} & \textbf{\textsc{Hyper-ES}} \\
        \midrule
        Qwen2.5-0.5B  & 0.0610 & 17.2486 & 0.3092 \\
        Qwen2.5-1.5B  & 0.1494 & 35.6538 & 0.4764 \\
        DeepSeek-1.5B & 0.0626 & 22.2137 & 0.2266 \\
        \bottomrule
    \end{tabular}}
\end{table}

\subsection{Additional Evidence for ES Drift}
\label{app:es_drift_evidence}

This appendix provides numerical diagnostics that correspond to the direct-ES failure modes illustrated in Figure~\ref{fig:figure15}. Table~\ref{tab:cosine_sim} shows that CMA-ES+LoRA moves in a direction almost orthogonal to the GRPO+LoRA update, while Table~\ref{tab:frobenius} shows that CMA-ES+LoRA also produces much larger adapter norms.

\newpage

\section{Detailed Methods}
\subsection{Detailed Merging Operator}
\label{app:merging_operator}

This appendix gives the full parameterization of the merging operator used in Section~\ref{sec:method}. For each group $\ell\in\{1,\ldots,L\}$, the search variable $\boldsymbol{z}_\ell=(\boldsymbol{a}_{\ell,1:N},\boldsymbol{b}_{\ell,1:N})$ contains a drop-rate logit and a mixing-weight logit for each direction $i\in\{1,\ldots,N\}$. They are decoded as
\begin{equation}
    p_i^{(\ell)}=\sigma(a_{\ell,i}),
    \qquad
    \omega_i^{(\ell)}=\exp(b_{\ell,i}).
    \label{eq:decode_app}
\end{equation}

\paragraph{LoRA task vectors.}
For a LoRA adapter with rank $r$ and scaling $\alpha$, the task vector is
\begin{equation}
    \Delta\theta_i^{(\ell)} = \frac{\alpha}{r} B_i^{(\ell)}A_i^{(\ell)}.
    \label{eq:lora_delta_app}
\end{equation}
These deltas are precomputed and cached before the search.

\paragraph{DARE sparsification.}
Given the decoded drop rate, we sample a Bernoulli mask and sparsify the task vector:
\begin{equation}
    \begin{aligned}
        \widetilde{\Delta\theta}_i^{(\ell)}
        &= \Delta\theta_i^{(\ell)}\odot M_i^{(\ell)}, \\
        M_i^{(\ell)}
        &\sim \mathrm{Bernoulli}(1-p_i^{(\ell)}).
    \end{aligned}
    \label{eq:dare_app_revised}
\end{equation}
Unlike canonical DARE, we omit the $1/(1-p)$ rescaling factor and instead control magnitude through a separate global scale.

\paragraph{TIES sign election and disjoint merge.}
For each scalar parameter $j$, the elected sign is
\begin{equation}
    \hat{s}_j^{(\ell)}
    = \mathrm{sign}\!\left(\sum_{i=1}^N \omega_i^{(\ell)}\widetilde{\Delta\theta}_{i,j}^{(\ell)}\right).
    \label{eq:sign_app}
\end{equation}
Only sign-agreeing entries contribute to the merged update:
\begin{equation}
    \overline{\Delta\theta}_{j}^{(\ell)} =
    \frac{\sum_i \omega_i^{(\ell)}\widetilde{\Delta\theta}_{i,j}^{(\ell)}\,\mathbf{1}[\mathrm{sgn}(\widetilde{\Delta\theta}_{i,j}^{(\ell)})=\hat{s}_j^{(\ell)}]}
    {\sum_i \omega_i^{(\ell)}\mathbf{1}[\mathrm{sgn}(\widetilde{\Delta\theta}_{i,j}^{(\ell)})=\hat{s}_j^{(\ell)}]}.
    \label{eq:disjoint_merge_app}
\end{equation}
The final update is scaled and applied as
\begin{equation}
    \theta_*^{(\ell)} = \theta_0^{(\ell)} + \gamma\overline{\Delta\theta}^{(\ell)},
    \qquad
    \gamma=\texttt{scale}\times N.
    \label{eq:apply_app}
\end{equation}

\subsection{Search Details}
\label{app:search_details}

\paragraph{Grid warm start.}
The warm-start stage restricts all groups to share the same drop rate and scale:
\begin{equation}
    p_i^{(\ell)}=p,\quad
    \omega_i^{(\ell)}=1,
    \quad
    \gamma=\texttt{scale}\times N,
    \qquad
    \forall i,\ell.
    \label{eq:grid_restriction_main}
\end{equation}
We sweep \texttt{scale} $\in\{1.0,1.1,\ldots,3.0\}$ and $p\in\{0.3,0.5,0.7,0.9\}$ across DARE/TIES flag combinations. In our experiments, the optimal global configurations are identified as $\texttt{scale}=2.7$ and $p=0.9$ for Qwen2.5-0.5B, $\texttt{scale}=2.6$ and $p=0.7$ for Qwen2.5-1.5B, and $\texttt{scale}=2.9$ and $p=0.5$ for DeepSeek-R1-Distill-Qwen-1.5B, all of which serve as the initialization for the subsequent CMA-ES stage.

\paragraph{CMA-ES.}
CMA-ES operates in the $L\times 2N$-dimensional merging space. At generation $t$, it samples $G$ candidates
\begin{equation}
    z_k^{(t)}
    \sim
    \mathcal{N}\!\left(
    m_z^{(t)},
    (\sigma_z^{(t)})^2 C_z^{(t)}
    \right),
    \quad k=1,\ldots,G.
    \label{eq:cma_z_sampling_main}
\end{equation}
Each candidate is decoded into a merged model and evaluated on the validation set:
\begin{equation}
    f_k^{(t)}
    =
    \mathcal{A}\!\left(\theta_*(z_k^{(t)})\right).
    \label{eq:cma_fitness_main}
\end{equation}
Before ranking candidates, we normalize fitness values within each generation:
\begin{equation}
    \widehat{f}_k^{(t)}
    =
    \frac{
        f_k^{(t)}-\mathrm{mean}_{j}(f_j^{(t)})
    }{
        \mathrm{std}_{j}(f_j^{(t)})+\varepsilon
    }.
    \label{eq:fitness_norm_main}
\end{equation}
The top-$\mu$ candidates under the normalized fitness update the mean:
\begin{equation}
    m_z^{(t+1)}
    =
    \sum_{k=1}^{\mu}
    c_k z_{k:G}^{(t)},
    \quad c_k>0,\quad \sum_{k=1}^{\mu}c_k=1.
    \label{eq:cma_z_update_main}
\end{equation}
We initialize $m_z^{(0)}$ from the best grid configuration, use $\sigma_z^{(0)}=0.3$, and set $C_z^{(0)}=I$. The default CMA-ES population size grows logarithmically with dimension,
\begin{equation}
    G = 4+\lfloor 3\ln(L\times 2N)\rfloor,
    \label{eq:cma_population_size}
\end{equation}
which is about $24$ in our main $L\times 2N=960$ search space. The final solution is taken from the converged distribution mean rather than the historical best candidate.

\clearpage

\subsection{Detailed Introduction to Evolution Strategies}
\label{app:es_details}

\subsubsection*{CMA-ES update rules}

CMA-ES~\citep{hansen2001cmaes} maintains a Gaussian search distribution
$\mathcal{N}(m^{(t)},(\sigma^{(t)})^2C^{(t)})$ and updates it through mean recombination, cumulative step-size adaptation, and covariance adaptation.

\paragraph{Mean update.}
After sorting $G$ offspring by fitness, the top-$\mu$ candidates define the next mean:
\begin{equation}
    m^{(t+1)} = \sum_{k=1}^{\mu} c_k x_{k:G}^{(t)},
    \qquad
    c_k>0,\quad \sum_{k=1}^{\mu}c_k=1.
    \label{eq:cma_mean_app}
\end{equation}

\paragraph{Step-size adaptation.}
The global step size is adapted with a cumulative evolution path:

\begin{align}
    p_\sigma^{(t+1)}
    &=
    (1-c_\sigma)p_\sigma^{(t)}
    \notag\\
    &\quad
    +\sqrt{c_\sigma(2-c_\sigma)\mu_w}\,
    (C^{(t)})^{-1/2}
    \notag\\
    &\quad
    \times
    \frac{m^{(t+1)}-m^{(t)}}{\sigma^{(t)}},
    \label{eq:cma_path_app}
    \\
    \sigma^{(t+1)}
    &=
    \sigma^{(t)}
    \exp\!\left(
    \frac{c_\sigma}{d_\sigma}
    \left(
    \frac{\|p_\sigma^{(t+1)}\|}{\chi_n}
    -1
    \right)
    \right),
    \label{eq:cma_sigma_app}
\end{align}

where $\chi_n=\mathbb{E}\|\mathcal{N}(0,I_n)\|$ is the expected norm of a standard Gaussian in $n$ dimensions.

\paragraph{Covariance update.}
The covariance matrix is updated with rank-one and rank-$\mu$ terms:
\begin{equation}
    \begin{aligned}
        C^{(t+1)}
        &=
        (1-c_1-c_\mu)C^{(t)}
        +c_1 p_c^{(t+1)}(p_c^{(t+1)})^\top  \\
        &\quad
        +c_\mu\sum_{k=1}^{\mu} c_k
        y_{k:G}^{(t)}(y_{k:G}^{(t)})^\top,
    \end{aligned}
    \label{eq:cma_cov_app}
\end{equation}
where $y_{k:G}^{(t)}=(x_{k:G}^{(t)}-m^{(t)})/\sigma^{(t)}$ and $p_c^{(t)}$ is the covariance evolution path. These updates explain why CMA-ES benefits from a compact search space: covariance estimation is meaningful only when the population provides enough signal relative to the dimension.


\subsection{Detailed Introduction to Model Merging}
\label{app:model_merging_details}

\paragraph{Task arithmetic.}
Task arithmetic~\citep{ilharco2022editing} computes task vectors $\Delta W_i=W_i-W_0$ and combines fine-tuned models as
\begin{equation}
    W^* = W_0+\sum_i \lambda_i \Delta W_i.
\end{equation}
It is simple and efficient, but it can suffer from interference when task vectors have conflicting parameter signs.

\paragraph{DARE.}
DARE~\citep{yu2024dare} reduces interference by randomly sparsifying task vectors:
\begin{equation}
    \begin{aligned}
        \widetilde{\Delta W}_i
        &=
        \frac{1}{1-p}\Delta W_i\odot m_i, \\
        m_i
        &\sim \mathrm{Bernoulli}(1-p).
    \end{aligned}
\end{equation}
The canonical $1/(1-p)$ factor preserves expected magnitude. \textsc{Hyper-ES} omits this fixed factor in the main merging operator and instead controls magnitude with the validation-selected global scale $\gamma$.

\paragraph{TIES-Merging.}
TIES-Merging~\citep{yadav2023tiesmerging} mitigates sign conflicts by electing an aggregate sign and averaging only sign-consistent entries. \textsc{Hyper-ES} applies this sign election after DARE sparsification and makes the election weight-aware through the CMA-ES-decoded mixing weights.


\newpage
\section{Detailed Experiment}

\subsection{Training Details}
\label{app:training_details}

\begin{table*}[t]
    \centering
    \renewcommand\arraystretch{0.9}
    \caption{GRPO training hyperparameters.}
    \small{\begin{tabularx}{0.82\textwidth}{l@{\hskip 0.8in}X}
            \toprule
            Parameter & Value \\
            \midrule
            \multicolumn{2}{c}{\textsc{Qwen2.5-0.5B/1.5B}} \\
            \midrule
            Training rollout max response length & 2048 tokens \\
            Rollout temperature           & 0.6 \\
            Max rollout sequences         & 512 \\
            PPO mini-batch size           & 256 \\
            PPO micro-batch per GPU       & 4 \\
            Group size $G$                & 8 \\
            Learning rate                 & $1\times10^{-6}$ \\
            KL coefficient $\beta$        & $1\times10^{-3}$ \\
            GPU memory fraction           & 0.30 \\
            Rollout backend               & vLLM \\
            \midrule
            \multicolumn{2}{c}{\textsc{DeepSeek-R1-Distill-Qwen-1.5B}} \\
            \midrule
            Training rollout max response length & 6144 tokens \\
            PPO micro-batch per GPU       & 1 \\
            Gradient checkpointing        & enabled \\
            Log-prob micro-batch          & 2 \\
            Other settings                & Same as above \\
            \bottomrule
    \end{tabularx}}
    \label{tab:grpo-config}
\end{table*}

\begin{table*}[htbp]
    \centering
    \renewcommand\arraystretch{0.9}
    \caption{LoRA adapter configuration used across all models.}
    \small{\begin{tabularx}{0.82\textwidth}{l@{\hskip 0.8in}X}
            \toprule
            Parameter & Value \\
            \midrule
            Rank $r$        & 32 \\
            Scaling $\alpha$ & 64 ($\alpha/r=2.0$) \\
            Target modules  & \texttt{q\_proj, k\_proj, v\_proj, o\_proj, gate\_proj, up\_proj, down\_proj} \\
            Storage dtype   & float16 \\
            \bottomrule
    \end{tabularx}}
    \label{tab:lora-config}
\end{table*}

\begin{table*}[htbp]
\centering
\small
\renewcommand{\arraystretch}{1.1}
\caption{Time efficiency breakdown of \textsc{Hyper-ES} across different stages. Execution times are measured primarily on 8 NVIDIA GeForce RTX 3090 (24GB) GPUs.}
\label{tab:time_efficiency}
\begin{tabular}{l r r}
    \toprule
    \textbf{Stage} & \textbf{Time (s)} & \textbf{Duration} \\
    \midrule
    \multicolumn{3}{l}{\textit{Baseline}} \\
    GRPO+LoRA (78 steps) & 15,314 & 4.25 h \\
    \midrule
    \multicolumn{3}{l}{\textit{Hyper-ES (Ours)}} \\
    Few-shot GRPO directions (Cumulative) & 16,747 & 4.65 h \\
    Few-shot GRPO directions (Parallelized) & 1,769 & 29.5 min \\
    Grid search (84 evals, DARE+TIES, scale 1--3) & 3,106 & 51.8 min \\
    CMA-ES (50 gens, extrapolated) & 6,750 & 1.88 h \\
    \midrule
    \textbf{Total (Ours, Cumulative)} & \textbf{26,603} & \textbf{7.39 h} \\
    \textbf{Total (Ours, Parallelized)} & \textbf{11,625} & \textbf{3.23 h} \\
    \bottomrule
\end{tabular}
\end{table*}

\subsubsection{GRPO Hyperparameters}
\label{app:training_grpo}
The detailed hyperparameter configurations used during the GRPO training phase across different model backbones are summarized in Table~\ref{tab:grpo-config}.

\subsubsection{LoRA Configuration}
\label{app:training_lora}
The specific structural parameters and training configurations applied to the LoRA adapters across all evaluated backbones are detailed in Table~\ref{tab:lora-config}.

\subsection{Time Efficiency}
\label{sec:time_efficiency}

Experiments in this paper are mainly conducted on 8 Nvidia 3090 24GB GPUs, and we also involve A100 and H100 GPUs for larger settings. Table~\ref{tab:time_efficiency} provides a detailed breakdown of the execution time consumed during each phase of \textsc{Hyper-ES} compared to the GRPO-LoRA baseline.

While maintaining an equivalent memory footprint to GRPO-LoRA, \textsc{Hyper-ES} significantly reduces the wall-clock execution time by bypassing the heavy communication overhead typical of distributed reinforcement learning. Taking our Qwen2.5-0.5B-Instruct experiment with 10 directions as a representative case, because the training of multiple few-shot LoRAs is entirely independent, they can be launched concurrently in an embarrassingly parallel fashion. This compresses the direction-generation phase to just $1,769$\,s ($29.5$\,min) for the longest single run. Combined with the subsequent decoupled grid search and low-dimensional CMA-ES, the end-to-end wall-clock time is only $3.23$\,h. This represents an approximate $24\%$ reduction in total training time compared to standard GRPO+LoRA ($4.25$\,h), demonstrating both the parallelizability and the overall efficiency of our framework.

\subsection{Performances over Stages}
\label{app:full_eval}

Table~\ref{tab:full_eval} reports validation and GSM8K test accuracy for individual difficulty-level adapters. Individual LoRA adapters each train on a single difficulty bucket; their results show that the benefit of \textsc{Hyper-ES} comes from merging the direction pool rather than selecting one direction.

\begin{table*}[t]
    \centering
    \footnotesize
    \setlength{\tabcolsep}{6pt}
    \begin{tabular}{lcccc}
        \toprule
        & \multicolumn{2}{c}{\textbf{Qwen2.5-0.5B}}
        & \multicolumn{2}{c}{\textbf{Qwen2.5-1.5B}} \\
        \cmidrule(lr){2-3} \cmidrule(lr){4-5}
        \textbf{Method}
        & \multicolumn{2}{c}{\textbf{In-domain GSM8K}}
        & \multicolumn{2}{c}{\textbf{In-domain GSM8K}} \\
        \cmidrule(lr){2-3} \cmidrule(lr){4-5}
        & \textbf{Val.} & \textbf{Test}
        & \textbf{Val.} & \textbf{Test} \\
        \midrule
        Base (no fine-tuning)   & 45.83 & 46.93 & 70.83 & 71.49 \\
        \midrule
        LoRA step 1             & 47.67 & 48.22 & 71.50 & 72.02 \\
        LoRA step 2             & 50.33 & 48.90 & 71.17 & 71.57 \\
        LoRA step 3             & 50.50 & 48.07 & 71.33 & 71.49 \\
        LoRA step 4             & 49.17 & 48.22 & 71.67 & 71.65 \\
        LoRA step 5             & 48.33 & 47.76 & 71.67 & 71.72 \\
        LoRA step 6             & 47.67 & 48.22 & 70.67 & 71.95 \\
        LoRA step 7             & 47.33 & 48.82 & 71.00 & 71.72 \\
        LoRA step 8             & 48.50 & 48.07 & 71.50 & 72.02 \\
        LoRA step 9             & 47.50 & 47.76 & 71.00 & 71.87 \\
        LoRA step 10+           & 49.17 & 47.76 & 71.17 & 71.42 \\
        \midrule
        GRPO+LoRA             & 58.17 & 50.57 & 76.67 & 74.22 \\
        CMA-ES+LoRA           & 53.83 & 45.94 & 76.00 & 73.54 \\
        \midrule
        \textsc{Hyper-ES}{} (grid only)       & 60.17 & 50.72 & 78.17 & 75.21 \\
        \textbf{\textsc{Hyper-ES}{} (CMA-ES)}
        & \textbf{63.83} & \textbf{51.18}
        & \textbf{79.67} & \textbf{75.51} \\
        \bottomrule
    \end{tabular}
    \caption{Complete per-method GSM8K results. Val. and Test denote held-out validation and full GSM8K test accuracy (\%) under greedy decoding ($T=0$).}
    \label{tab:full_eval}
\end{table*}

\subsection{Additional Results}
\label{app:additional_results}
\subsubsection{Multi-seed Results}
\label{app:multiseed}

We conduct three independent runs on Qwen2.5-0.5B-Instruct by using different random seeds for dataset splitting, short GRPO training, direction-pool construction, and CMA-ES search. Therefore, a new direction pool is constructed in each run.

\begin{table*}[t]
    \centering
    \small
    \caption{Multi-seed results on Qwen2.5-0.5B-Instruct. We report mean $\pm$ standard deviation over three independent runs.}
    \label{tab:multiseed}
    \begin{tabular}{lccccc}
        \toprule
        Method & GSM8K & SVAMP & MultiArith & GSM-Hard & Avg. \\
        \midrule
        GRPO+LoRA & $50.92\pm0.68$ & $60.89\pm0.38$ & $94.82\pm0.85$ & $17.51\pm0.20$ & $56.04\pm0.29$ \\
        \textsc{Hyper-ES} & $50.85\pm0.39$ & $\mathbf{63.67\pm1.45}$ & $93.70\pm0.32$ & $\mathbf{18.27\pm0.33}$ & $\mathbf{56.62\pm0.44}$ \\
        \bottomrule
    \end{tabular}
\end{table*}

Compared with GRPO+LoRA, \textsc{Hyper-ES} improves the average score from $56.04\pm0.29$ to $56.62\pm0.44$. The modest standard deviation across three independent runs, further demonstrates the stability of \textsc{Hyper-ES} across different direction-pool constructions.

\subsubsection{Additional Baselines}
\label{app:additional_baselines}

We further compare \textsc{Hyper-ES} with tinyGRPO~\citep{openthought2025tinygrpo}, CABS~\citep{yang2025cabs}, and WUDI-Merging~\citep{cheng2025whoever}. We apply the model-merging baselines to the same direction pool used by \textsc{Hyper-ES}.

\begin{table*}[t]
    \centering
    \small
    \caption{Comparison with additional baselines on Qwen2.5-0.5B-Instruct.}
    \label{tab:additional_baselines}
    \begin{tabular}{lccccc}
        \toprule
        Method & GSM8K & SVAMP & MultiArith & GSM-Hard & Avg. \\
        \midrule
        tinyGRPO~\citep{openthought2025tinygrpo} & 47.99 & 53.67 & 82.78 & 16.38 & 50.21 \\
        CABS~\citep{yang2025cabs} & 47.99 & 58.67 & 90.56 & 16.53 & 53.44 \\
        WUDI-Merging~\citep{cheng2025whoever} & 51.10 & 61.67 & 92.22 & 17.29 & 55.57 \\
        \textsc{Hyper-ES} & \textbf{51.18} & \textbf{65.33} & \textbf{93.89} & \textbf{18.12} & \textbf{57.13} \\
        \bottomrule
    \end{tabular}
\end{table*}

\textsc{Hyper-ES} achieves the best result on all four benchmarks. It improves the average score by 6.92 points over tinyGRPO and by 1.56 points over the strongest model-merging baseline.

\subsubsection{Out-of-domain Code Generation}
\label{app:code_results}

To evaluate the generalizability of \textsc{Hyper-ES} to code-generation tasks, we conduct experiments on MBPP~\citep{austin2021program} and HumanEval~\citep{chen2021evaluating}. The evaluation covers the MBPP test set with 500 problems and the HumanEval test set with 164 problems, using greedy decoding with $T=0$ and a maximum output length of 1,024 tokens.

\begin{table*}[t]
    \centering
    \small
    \caption{Out-of-domain generalization results on code-generation tasks.}
    \label{tab:code_results}
    \begin{tabular}{lccc}
        \toprule
        Method & MBPP & HumanEval & Avg. \\
        \midrule
        Qwen2.5-0.5B + GRPO+LoRA & \textbf{36.40} & 31.71 & 34.06 \\
        Qwen2.5-0.5B + \textsc{Hyper-ES} & 34.00 & \textbf{34.76} & \textbf{34.38} \\
        Qwen2.5-1.5B + GRPO+LoRA & 47.60 & 29.27 & 38.44 \\
        Qwen2.5-1.5B + \textsc{Hyper-ES} & \textbf{49.60} & \textbf{43.29} & \textbf{46.45} \\
        \bottomrule
    \end{tabular}
\end{table*}

\textsc{Hyper-ES} achieves higher average code-generation performance on both models. On Qwen2.5-1.5B, the average score increases from 38.44 to 46.45. These results demonstrate that \textsc{Hyper-ES} generalizes effectively to code-generation tasks.

\section{Baselines \& Datasets \& License} \label{app:baselines_licenses}

\subsection{Baselines}

Based on \textbf{Qwen2.5-0.5B/1.5B-Instruct}~\citep{qwen2025qwen25} and \textbf{DeepSeek-R1-Distill-Qwen-1.5B}~\citep{guo2025deepseek}, we include a fine-tuning baseline (GRPO+LoRA), direct CMA-ES+LoRA, model-merging baselines such as Average Merge~\citep{ilharco2022editing} and DARE+TIES~\citep{yu2024dare,yadav2023tiesmerging}, and \textsc{Hyper-ES}.

\paragraph{Fine-tuning Baseline} GRPO+LoRA trains a single LoRA adapter on the full training corpus for the same total number of gradient steps as the combined direction-generation budget. We implement it using VERL. The training prompt is given in Appendix~\ref{app:prompts}. This baseline controls for the total GRPO compute invested by \textsc{Hyper-ES} and isolates the effect of the evolutionary merging stage.

\paragraph{Model Merging Baselines} Average Merge uniformly averages all $N$ LoRA task vectors. DARE+TIES applies random sparsification~\citep{yu2024dare} followed by sign-conflict resolution~\citep{yadav2023tiesmerging} at fixed drop rates. These baselines share the same direction pool as \textsc{Hyper-ES} but use fixed or grid-searched hyperparameters rather than CMA-ES-optimized per-layer coefficients.

\subsection{Datasets}

We adopt GSM8K~\citep{cobbe2021gsm8k} for in-domain evaluation, and GSM-Hard~\citep{gao2023pal}, SVAMP~\citep{patel2021nlp}, and MultiArith~\citep{roy2015solving} for out-of-domain evaluation on Qwen2.5 models. For DeepSeek-R1-Distill-Qwen-1.5B, we additionally report MATH-500~\citep{lightman2024verify}, AMC23. For code-generation evaluation, we use the MBPP test set~\citep{austin2021program} with 500 problems and HumanEval~\citep{chen2021evaluating} with 164 problems.

\subsection{License}
\label{app:license}

For base LLMs, datasets, and frameworks, we list their licenses in Table~\ref{tab:licenses}.

\begin{table*}[t]
    \centering
    \setlength{\tabcolsep}{1.1mm}
    \renewcommand\arraystretch{1.2}
    \caption{A summary of licenses.}
    \resizebox{\textwidth}{!}{
        \begin{tabular}{llll}
            \toprule
            Resource & Type & License & URL \\ \midrule
            Qwen2.5-0.5B-Instruct & Base LLM & Qwen License & \url{https://huggingface.co/Qwen/Qwen2.5-0.5B-Instruct} \\
            Qwen2.5-1.5B-Instruct & Base LLM & Qwen License & \url{https://huggingface.co/Qwen/Qwen2.5-1.5B-Instruct} \\
            DeepSeek-R1-Distill-Qwen-1.5B & Base LLM & MIT License & \url{https://huggingface.co/deepseek-ai/DeepSeek-R1-Distill-Qwen-1.5B} \\ \midrule
            VERL & Training framework & Apache-2.0 & \url{https://github.com/volcengine/verl} \\
            vLLM & Inference engine & Apache-2.0 & \url{https://github.com/vllm-project/vllm} \\ \midrule
            GSM8K, GSM-Hard, SVAMP & Dataset & MIT License & \url{https://github.com/openai/grade-school-math} \\
            MultiArith & Dataset & Available Online & \url{https://github.com/sroy9/equationparsing} \\
            MATH-500 & Dataset & MIT License & \url{https://github.com/openai/prm800k} \\
            GSM8K-AUG & Dataset & Available Online & \url{https://huggingface.co/datasets/whynlp/gsm8k-aug} \\
            DeepScaleR-Dataset & Dataset & MIT License & \url{https://huggingface.co/datasets/agentica-org/DeepScaleR-Preview-Dataset} \\ 
            MBPP & Dataset & CC BY 4.0 & \url{https://github.com/google-research/google-research/tree/master/mbpp} \\
            HumanEval & Dataset & MIT & \url{https://github.com/openai/human-eval} \\ \bottomrule
    \end{tabular}}
    \label{tab:licenses}
\end{table*}

\subsection{LLM Usage}

This paper employs Claude Code for experiments and writing, but we are responsible for the contents.


\end{document}